\documentclass{article} 
\usepackage{sty/preamble_arxiv}

\usepackage{outlines}
\usepackage{tcolorbox}
\usepackage{graphicx}
\usepackage{subcaption}

\usepackage[utf8]{inputenc}
\usepackage[style=alphabetic, sortcites, maxbibnames=99]{biblatex}

\definecolor{darkgreen}{RGB}{0, 100, 0} 
\definecolor{darkred}{RGB}{139, 0, 0} 

\title{Black-Box Privacy Attacks on Shared Representations in Multitask Learning}

\author[1]{John~Abascal}
\author[1]{Nicol\'as~Berrios}
\author[1]{Alina~Oprea}
\author[1]{Jonathan~Ullman}
\author[2,3]{Adam~Smith}
\author[3]{Matthew~Jagielski}
\affil[1]{Khoury College of Computer Sciences, Northeastern University}
\affil[2]{Department of Computer Science, Boston University}
\affil[3]{Google DeepMind, Advisor Only}

\date{}
\begin{document}

\maketitle


\begin{abstract}
\emph{Multitask learning} (MTL) has emerged as a powerful paradigm that leverages similarities among multiple learning tasks, each with insufficient samples to train a standalone model, to solve them simultaneously while minimizing data sharing across users and organizations. MTL typically accomplishes this goal by learning a \emph{shared representation} that captures common structure among the tasks by embedding data from all tasks into a common feature space. Despite being designed to be the smallest unit of shared information necessary to effectively learn patterns across multiple tasks, these shared representations can inadvertently leak sensitive information about the particular tasks they were trained on.

In this work, we investigate what information is revealed by the shared representations through the lens of inference attacks. Towards this, we propose a novel, \emph{black-box task-inference} threat model where the adversary, given the embedding vectors  produced by querying the shared representation on samples from a particular task, aims to determine whether that task was present when training the shared representation. We develop efficient, purely black-box attacks on machine learning models that exploit the dependencies between embeddings from the same task without requiring shadow models or labeled reference data. We evaluate our attacks across vision and language domains for multiple use cases of MTL and demonstrate that even with access only to fresh task samples rather than training data, a black-box adversary can successfully infer a task's inclusion in training. To complement our experiments, we provide theoretical analysis of a simplified learning setting and show a strict separation between adversaries with training samples and fresh samples from the target task's distribution.
\end{abstract}



\section{Introduction}

\emph{Multitask learning} (MTL) has emerged as a powerful paradigm that leverages similarities among multiple learning tasks, each with insufficient samples to train a standalone model, to solve them simultaneously while minimizing data sharing across multiple entities, such as users and organizations.
MTL accomplishes this goal by learning a \emph{shared representation} that captures common structure between the tasks. Concretely, a shared representation could be a neural network that learns a mapping of the data from all tasks into a shared feature space, where similar data points across tasks cluster together, and task-specific output layers (or heads) operate on these embeddings—that is, the outputs of the shared representation—to make predictions. MTL methods have shown remarkable success in various domains, including computer vision~\cite{girshick2015fastrcnn}, natural language processing~\cite{collobert2008multitasklanguage}, federated learning~\cite{smith2017federated, hanzely2021federatedmixture, mansour2020personalization, ghosh2020clusteredfederated}, drug discovery~\cite{ramsundar2015drugdiscovery}, and financial forecasting~\cite{ghosn1996finance}. For example, MTL can be used to personalize image classification models by learning a shared across many users and locally adapting to each user's small, on-device photo library.


Despite being designed to capture only generic patterns that can be applied to several downstream tasks, these shared representations can inadvertently leak sensitive information about specific tasks, or underlying data distributions, that they were trained on. The privacy risks that arise are of particular concern when data from several sensitive entities, such as individual users, private organizations, or data silos, are jointly used for training across multiple tasks or non-uniform data distributions. Here, shared representations are often the minimum piece of information that each entity must contribute in order to achieve strong generalization while only contributing a limited number of samples. This point is highlighted in prior work on privacy-preserving collaborative learning~\cite{shen2023sharerep}, where sharing \emph{only} the representation, not task-specific layers, demonstrably improves convergence rates and model performance when learning over multiple parties. 

In our work, we study privacy attacks at the level of entire tasks and broadly investigate the privacy risks of jointly learning over multiple tasks by mounting these attacks on the smallest unit of information required to jointly train a model, the shared representation. In particular, we attempt to determine whether, and to what extent, an adversary can infer a given task's inclusion in training given black-box query access to (i.e. the ability to sample embedding vectors from) the shared representation. While there is prior work that explores privacy attacks on MTL, this work is limited to sample-level membership inference and model extraction~\cite{yan2024miamtl} where the adversary can make queries to the task-specific heads of the multitask model and train reference models. 

Towards this goal, we propose a new threat model called \emph{task-inference} and develop black-box attacks with minimal adversarial knowledge to infer the inclusion of a task in MTL. Unlike membership-inference \cite{homer-2008-mi,SankararamanOJH09,dwork-2015-trace,shokri-2016-mi}, task-inference generalizes individual privacy attacks from the sample level to the task level, where the adversary aims to determine the presence of an entire target task, rather than a particular sample, in the training set. Critically, our threat model only assumes that the adversary has access to samples from the target task's distribution. We identify two variants of our threat model, one with a \textbf{strong} adversary who has the specific samples used to train the shared representation and another with a \textbf{weak} adversary who only has independent samples from the target task's distribution. Our experiments demonstrate a notable separation between these two threat models in terms of their attack capabilities,  and we provide analysis of a simplified learning setting to help explain our empirical findings in machine learning. We observe that while both adversaries can mount successful attacks, having access to training samples provides a sizeable advantage. Furthermore, in contrast to prior work \cite{shokri-2016-mi, carlini-2022-lira, 2021-liu-encodermi, chen2023faceauditor, chaudhari-2022-snap, abascal2023tmi, yan2024miamtl}, neither of the task-inference adversaries require auxiliary datasets to train reference, or shadow, models to calibrate their attack. We instead mount attacks by using the key observation that representation models tend to output codependent embedding vectors for samples that come from the same task. 

Because MTL is a general framework that enables learning high utility models by leveraging similarities between related problems, there are several different ways in which tasks can be defined. In this work, we study privacy leakage in two use cases of MTL which we believe are representative of common applications: MTL across many individuals, or users, for \textbf{personalization} and MTL for solving \textbf{multiple related learning problems}. A canonical use of MTL for personalization is photo tagging (i.e. detecting items, attributes, or people in images), where each task corresponds to an individual's mobile device with limited data who cannot train an model on their own. Thus, they learn a shared representation for a common learning problem with other participants in MTL and locally personalize a task-specific layer to their data. In contrast, MTL could also be used to solve several distinct learning problems with limited data but related structure, such as detection of multiple facial attributes in images or multiple topics in text. Here, each task can be seen as a different learning problem with limited data and a unique label space, but each task benefits from learning the structure of other tasks (e.g. facial attributes or topics can be correlated). We observe that the privacy risks that arise from our threat model depend on the specific application of MTL. When using MTL for personalization and tasks are defined as people, inferring inclusion of a task in training can be seen as analogous to membership~\cite{homer-2008-mi} or user-inference~\cite{kandpal2023userinference}. In this setting, the adversary composes the model's outputs over several samples belonging to an individual to determine if any of their data appeared in the dataset. If MTL is used to simultaneously solve multiple learning problems, inferring a task's inclusion poses a similar risk to inferring whether a property~\cite{ateniese-2015-property_inference, hartmann-2023-distributioninference, chaudhari-2022-snap}, or subpopulation that satisfies a common labeling, was present at all in the training dataset. Here, we observe that, similar to property inference, an adversary uses several samples sharing common attributes in the input space (e.g. all text samples discussing the topic "Python") to determine either if, or to what extent, samples satisfying these attributes were included in the training dataset. 

We highlight that the task-inference threat model we propose interpolates those of well-studied privacy attacks on machine learning, and it induces risks comparable to membership, user, and property inference depending on factors like the number of samples per task and the particular definition of tasks. While our analysis in Section~\ref{sec:tracing} offers broad intuition on the efficacy of privacy attacks on groups and non-i.i.d. data, connecting them to sample-level attacks, MTL is a particularly apt regime for studying leakage at the granularity of multiple related samples.


We comprehensively evaluate our black-box task-inference attacks across vision and language datasets, including CelebA~\cite{liu-2015-celeba} and Stack Overflow~\cite{stackoverflow}, where shared representations are trained using MTL for both personalization across several users and for simultaneously solving multiple learning problems. We also study the factors that lead to task-inference leakage by observing how attack success varies in the model's generalization gap and conducting ablation studies on synthetic data across parameters such as the total number of tasks in the dataset, the embedding dimension, and the number of samples per task. Our findings demonstrate that a purely black-box adversary can successfully infer inclusion of tasks in MTL using only query access to the shared representation, even with \textbf{weak} access to training data samples, reinforcing the generality of our threat model and the separation between \textbf{strong} and \textbf{weak} adversaries observed through our theoretical analysis of tracing attacks.

\subsection*{Our Contributions} We summarize the main contributions of our work as follows:

\begin{enumerate}[labelsep=*,leftmargin=15pt]
    \item We formulate novel, black-box threat models for privacy attacks on datasets which are composed of multiple tasks or distributions, focusing on shared representations learned with MTL. Our proposed threat models distinguish between adversaries with \textbf{strong} versus \textbf{weak} access to the shared representation's training data.
    \item We analyze our threat models in the context of tracing attacks for mean estimation over mixtures of Gaussian distributions, and demonstrate a notable separation between the \textbf{strong} and \textbf{weak} adversaries in terms of achievable attack success rates. In addition, we highlight the connections between our tracing attack on tasks and tracing attacks on individual samples. 
    \item We construct two efficient, purely black-box privacy attacks for machine learning models trained over multiple tasks, demonstrating that an adversary can infer the inclusion of a task when they only have \textbf{weak} access to the training data, but can mount significantly more powerful attacks with \textbf{strong} access.
    \item We extensively evaluate both of our attacks on shared representations in machine learning by running experiments in both the vision and language domains for two representative use-cases of MTL, and we empirically study the factors that lead to privacy leakage when learning shared representations over several tasks.
\end{enumerate}

\section{Background and Related Work}

We provide background on multitask learning, related work on existing privacy attacks, and federated learning.

\subsection{Multitask Learning}

The goal of MTL (see~\cite{caruana-1997-multitask} for an early survey) is to learn jointly over several related tasks, often with sparse data, by exploiting shared features between them. A \textit{task} $\tau$ is defined as a distribution over samples in $\mathcal{X} \times \mathcal{Y}$ (e.g. $d$-dimensional real vectors that correspond to binary labels). In this work, we assume that the tasks for some learning problem are drawn from a \textit{task distribution}, $\mathcal{Q}$. To leverage common features between tasks, we can learn a shared representation $h : \mathcal{X} \to \mathcal{Z}$, that maps samples in $\mathcal{X}$ to a lower dimensional space, $\mathcal{Z}$, where the representation vectors, called \emph{embeddings}, capture the shared structure of the tasks. Mapping these samples to a lower dimensional space simplifies the learning problem, allowing us to use simple, linear classifiers $g_i: \mathcal{Z} \to \{0,1\}$. Given data samples from $T$ tasks $\vec{\tau} = (\tau_1,  \dots,  \tau_T)$ and a loss function $\mathcal{L}$, learning the shared representation $h$ can be written as the following optimization problem:
\begin{align*}
     \min_{{h_\theta \in \mathcal{H}}} 
     \left( 
     \min_{g_{\beta_1},\dots,g_{\beta_{T}}} \sum_{i=1}^{\mathcal{T}} \mathcal{L}(g_{\beta_{i}} \circ h_{\theta}, \: \tau_i) 
     \right)
\end{align*}

The shared representation $h$ often takes the form of a neural network that maps an example $x \in \mathcal{X}$ to an embedding vector $z \in \mathcal{Z}$ with task-specific classifiers $\{g_1, ..., g_{T}\}$ being linear classifiers applied to the embedding vector.  Thus, the goal is to learn an $h$ such that it maps data samples to embedding vectors that are linearly separable. 

\subsubsection{Defining Tasks}

In this work, we consider two MTL settings that are representative of common scenarios: (1) \textbf{personalization}, where each \emph{user} or \emph{person} represents a task with a personalized objective (e.g., multilabel face detection, recommendation, etc.) learned from sparse data by leveraging the shared representation trained over all users; and (2) \textbf{multiple learning problems}, where the tasks correspond to distinct classification problems (e.g.,  detection of facial attributes, binary topic classification, etc.) that share a similar underlying structure. We choose these two settings because they highlight how our proposed threat model (Section~\ref{sec:threatmodel}) captures and abstracts the privacy threats posed by existing attacks at the individual level~\cite{shokri-2016-mi}, user level~\cite{kandpal2023userinference}, and property level~\cite{ateniese-2015-property_inference}. 
 

\subsubsection{Connections to Distributed and Collaborative Learning}

The most ubiquitous approaches to collaborative learning are based on federated learning~\cite{mcmahan2017FLoriginal} (FL). FL is a collaborative machine learning framework where a central, trusted server coordinates model training across several parties, such as users or silos~\cite{huang2022crosssiloFL}. In a simple FL setting, the server publishes an initial machine learning model to several devices and each device computes model updates locally on its data. The parties then send the updated models or the raw updates themselves to the trusted server to be aggregated, and the process repeats until convergence. While FL can be used as a tool to learn a single task over several parties and publish the shared model, many popular constructions of FL aim to collaboratively learn models that can be locally adapted to an individual party's data during training in a multitask fashion~\cite{smith2017federated}.

In fact, there are a multitude of prior works that study the connections between FL and MTL~\cite{yu-2022-FLadaptation, tan-2023-personalizedFL, hu2023privatemtl, fallah2020personalizedfederatedlearningmetalearning, mansour2020personalization}, many of which focus on techniques to personalize models on users' data by leveraging shared representations. Using collaborative learning for personalizing to users' data has also seen adoption in industry research~\cite{msr2022PersonalizedFL}, making it imperative to understand whether there are privacy risks incurred by having minimal access to the shared portion of the collaboratively learned model.

\subsection{Privacy Attacks on ML Models}

While overparameterized deep learning models are known to memorize individual training data points, even in settings where the labeling is random~\cite{zhang-2017-rethinking}, prior work has shown that models capable of such memorization still achieve good generalization on unseen data~\cite{belkin-2019-reconciling, feldman-2020-longtail, brown-2021-irrelevant}. This fact opens multiple paths for an inference attack adversary to learn sensitive information about a machine learning model's training data. In particular, these models can leak membership information about the individual points they memorized~\cite{shokri-2016-mi}, and they can leak information about the distributions of subpopulations within their training datasets~\cite{ateniese-2015-property_inference, kandpal2023userinference}.

\subsubsection{Membership-Inference Attacks}
Membership inference attacks (MIAs)~\cite{homer-2008-mi,SankararamanOJH09,dwork-2015-trace,shokri-2016-mi} seek to determine whether or not a given data record was present in a machine learning model's training dataset. These attacks are the most widely studied in the privacy literature as they represent a fundamental form of privacy leakage that has direct connections to the definition of differential privacy~\cite{dwork-2006-dmns}. Because of this, MIAs have been adapted to several machine learning settings and architectures, such as encoder models~\cite{liu2021encodermi, song-2020-embeddingleakage}, large language models~\cite{zhang2024mink, duan2024membershipinferenceattackswork}, and federated learning~\cite{suri2023subjectmembershipinferenceattacks, nasr-2019-federatedmia}, and they are often used to audit the empirical privacy leakage of machine learning models~\cite{song-2018-auditingmi, ye-2022-enhanced_mi, jagielski2020auditing}. While the privacy risks that such attacks pose on their own is subtle, mounting accurate MIAs is often a necessary step for performing powerful attacks such as training data extraction~\cite{carlini-2020-extraction_llm, carlini-2023-extraction_diffusion}. 

Additionally, there are studies on membership-inference attacks in the setting where the adversary has some knowledge of the distributions that comprise the training dataset~\cite{humphries-2023-dependencies, rezaei-2022-subpopulationbasedmia}. These works find that distinguishing between a member target sample and known non-member samples from the same subpopulation is an easier adversarial objective than distinguishing between a member target sample and a i.i.d. samples from the training dataset's distribution. The former study~\cite{humphries-2023-dependencies} investigates how this fact changes differential privacy guarantees, while the latter work~\cite{rezaei-2022-subpopulationbasedmia} leverages this additional knowledge to mount efficient membership-inference attacks.


\subsubsection{Property-Inference Attacks} In contrast to individual privacy attacks, property (or distribution) inference attacks~\cite{ateniese-2015-property_inference} aim to determine global attributes of the dataset, such as the proportions of subpopulations. These attacks were originally formalized as a distinguishing test between two worlds in which the subpopulation of interest comprises either $t_0$ or $t_1 > t_0$ fraction of the training data set. In recent years, the property inference threat model has been extensively studied in various settings in recent years with black-box and white-box access~\cite{suri2023dissectingdistributioninference, zhou2022propinfGANS, ganju2018propinfpermutation, hartmann-2023-distributioninference}, under data poisoning~\cite{mahloujifar2022propertypoisoning, chaudhari-2022-snap}, and within collaborative learning~\cite{xu2020subjectpropinf, melis-2018-exploitingunintendedcollab}. Extensions of the original property inference threat model, known as \emph{property existence attacks}~\cite{chaudhari-2022-snap} or \emph{distributional membership inference}~\cite{hartmann-2023-distributioninference}, have shown that, with and without data poisoning, respectively, an adversary can determine whether a property appeared at all in the data, rather than determining its frequency in the dataset. This can be seen as a special case of property inference, where $t_{0} = 0$ and $t_1 > 0$.

\subsubsection{Attacks on Representation Models} There are multiple works studying privacy leakage from representation models, which are trained to produce rich embedding vectors that can be applied to several downstream learning tasks. In two of the the seminal works on privacy leakage from embedding models~\cite{song-2020-overlearning, song-2020-embeddingleakage}, the authors find that unsupervised representation models can leak information about the training set that is uncorrelated with the learning task. The latter of these two works investigates a variety of attacks, such as membership-inference and inversion, on word embeddings. In~\cite{liu2021encodermi}, the authors observe that vision encoders trained with contrastive loss produce embedding vectors that are robust to augmentations, such as rotations and horizontal flips, of images that were included in training. By leveraging this fact, along with reference (or shadow) models to calibrate the attack, they mount successful membership-inference attacks on encoders.

\subsubsection{Privacy Attacks on Groups} 
While there are countless studies on individual privacy attacks, few works have studied attacks with coarser granularity than MIA and extraction but finer than property inference. Such privacy attacks can leverage the the composition of privacy leakage over correlated groups of sensitive samples, such as multiple messages or logs from a single device~\cite{mcmahan-2018-learningdplm}, to amplify privacy attacks on a single individual. One study on dual encoder models~\cite{song-2020-embeddingleakage} proposed attacks to infer membership of multiple sentences in a language model's context at once by measuring the similarity between context sentences in representation space. A similar work on facial detection models~\cite{chen2023faceauditor} shows that an adversary can infer inclusion of a person's contribution of images to the model's training data. Both of these works use large reference datasets to calibrate the adversary's attack using learned similarity scores and shadow models, respectively, from known members and non-members.

Recent work has started to explore user or distribution-level attacks with weaker assumptions. For example,~\cite{kandpal2023userinference} explores the privacy risks associated with fine-tuning generative large language models on data that are user-partitioned. This work uses similar techniques to the membership-inference literature, but generalizes the distinguishing test to the level of a given user's entire contribution to a dataset. This is accomplished by first observing the prediction confidences of a fine-tuned language model on a user's sentences that were not seen during training which assumes the same data access as our \textbf{weak} adversary. Their attack then calibrates the confidence scores against the prediction confidences of a pretrained reference model for which the user's contributions are assumed to be non-members. One related work to user level attacks \cite{suri2023subjectmembershipinferenceattacks} investigates leakage of subjects, or individuals, in cross-silo federated learning, where the adversary uses query access to the model after each federation round to determine a subject's inclusion in the training data.  

In a similar vein, the extension of property inference proposed by \cite{chaudhari-2022-snap} and \cite{hartmann-2023-distributioninference} aims to infer the membership of subpopulations in general rather than a particular user. The adversaries considered in both works have a similar level of access to the \textbf{weak} adversary. However, these works assume that the adversary has a sufficient number of samples from the target subpopulation and computational power to train several reference models. Recent related work has also investigated an adversary's ability to infer membership at the resolution of entire datasets in text corpora used to train large language models~\cite{maini-2024-datasetinference, ren-2025-datasetself}.

\begin{figure*}[t!]
    \centering
    \includegraphics[width=0.8\linewidth]{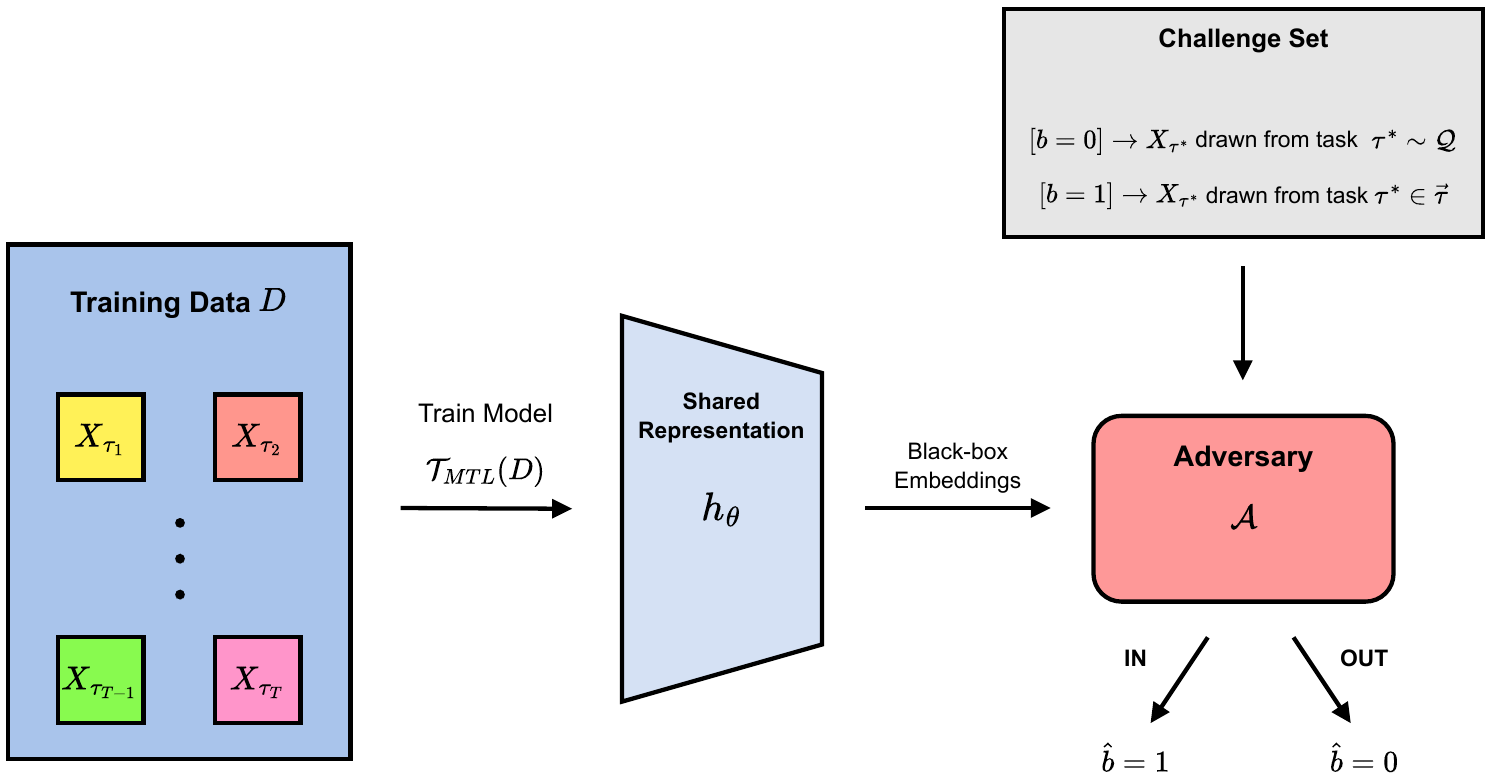}
    \caption{Security Game for Task Inference Attacks}
    \label{fig:threat-model}
\end{figure*}

\section{Threat Model} \label{sec:threatmodel}

Our goal is to assess whether shared representations reveal information about the specific tasks that they were trained on and to what degree they reveal this information. Towards this, we study privacy leakage in shared representations through the lens of inference attacks~\cite{homer-2008-mi, shokri-2016-mi, ateniese-2015-property_inference}. In this setting, there is an MTL model that is trained on several tasks simultaneously to learn a shared representation, or encoder, and individual task layers. Queries can then be made to the encoder to receive representation vectors, or embeddings, for any given input. A challenge task is drawn from the same distribution as the tasks used to train the MTL model. Our adversary uses their query access to the shared representation, along with data drawn from the challenge task, to infer whether the challenge task was included while training the MTL model. In this work, we study this leakage from a purely black-box perspective, only assuming query access and no knowledge of the underlying task distribution to train shadow models as is common in prior works~\cite{carlini-2022-lira}. Additionally, we reveal a separation between a \emph{strong adversary} that receives a set of the actual training data from the challenge task, and a \emph{weak adversary} that receives non-training data drawn from the same distribution as the challenge task. We can describe our threat model using the following security game between a \textit{challenger} and an \textit{adversary}:




    













 \begin{enumerate}[labelsep=*,leftmargin=15pt]
     \item[]  \begin{center}
         \textbf{Task Inference Security Game}
     \end{center} 
     \item The challenger receives $T$ tasks $\vec{\tau} = ( \tau_1, \dots, \: \tau_T )$, drawn from a task distribution $\mathcal{Q}$. For each of the tasks, $\tau_i$, the challenger is given a batch of samples $X_i$ and concatenates the $T$ batches into a dataset $D = \{X_{\tau_1}, \dots, X_{\tau_T} \}$.
     \item The challenger trains a \textit{shared representation} $h_\theta~\gets~\mathcal{T}_{\mathit{MTL}}(D)$ by simultaneously learning \emph{task-specific} models that share $h_\theta$.
     \item The challenger randomly selects $b \in \{0, 1\}$. If  $b = 0$, the challenger samples a challenge task $\tau^*$ from $\mathcal{Q}$ uniformly at random, such that $\tau^* \notin \vec{\tau}$. Otherwise, the challenger samples $\tau^*$ from $\vec{\tau}$ uniformly at random. 
     \item The challenger sends a batch of samples, $X^*$, drawn from the challenge task, $\tau^*$.
     \item The adversary, using the batch of samples and black-box access to $h_\theta$, guesses a bit $\hat{b} \gets \mathcal{A}(h_\theta(X_{\tau^*}))$.
     \item The adversary wins if $\hat{b} = b$ and loses otherwise.
 \end{enumerate}

In this security game the adversary only requires samples drawn from the challenge task rather than the specific training samples from $D$. We claim that a task-inference adversary does not necessarily need data from a member of the training dataset to perform an inference attack on a task that contains their data. Semantically, if a multitask model is trained over sparse data from several tasks, our proposed adversary can use \textit{fresh data} from a given task's data distribution to infer whether the task was included in the training dataset. We use the following terms to define this distinction: A task-inference adversary is \textbf{strong} if they have access to training samples when $b=1$ or \textbf{weak} if they \emph{do not} have access to training samples when $b=1$. In either setting, a task-inference adversary can mount attacks to infer inclusion of a task in training. To support this claim, we show a separation between the \textbf{strong} and \textbf{weak} adversaries in Section~\ref{sec:mean_est} and empirically verify this separation in Section~\ref{sec:eval}. 

Our proposed threat model interpolates those of existing property existence attacks~\cite{chaudhari-2022-snap} and user inference attacks~\cite{kandpal2023userinference}. Suppose that the challenger trains a MTL model over many users, such that each task belongs to a unique user.  In this setting, the adversary is distinguishing between worlds where a MTL model was or was not trained on a specific user's data. Detecting the inclusion of a task in MTL training yields the same privacy violation as user inference. If the task only includes one sample per user, the attack reduces to membership inference where each sample belongs to a member. When MTL is used to train tailored classifiers for multiple learning problems, each with limited data, our task-inference adversary is distinguishing between worlds where the model was or was not trained on a particular pair of samples and labels. In other words, the adversary is attempting to determine if samples that satisfy the labeling for a given task were included in MTL training. Here, we can think of tasks as properties or distributions within the training dataset since the inclusion or noninclusion of a task determines whether or not a particular distribution of labels over samples was learned during MTL training. Thus, inferring the inclusion of a task yields a similar privacy violation to property existence~\cite{chaudhari-2022-snap}, where an adversary can determine the subpopulations that compose the dataset.

Additionally, we note that the adversary described in the security game operates in a black-box manner, only requiring query access to the representation $h_\theta$. In contrast to prior work on membership and property inference, our adversary \emph{does not} require shadow models nor sampling access to the underlying data distribution~\cite{carlini-2022-lira, mahloujifar2022propertypoisoning,shokri-2016-mi, liu2021encodermi, abascal2023tmi}. Our task-inference adversary's access can be viewed as similar to that of the membership inference adversary described in \cite{liu2021encodermi}. In this work, the adversary makes black-box queries to a representation model and receives embedding vectors for multiple correlated samples which are then used to mount an attack. In contrast, the adversary proposed by~\citet{liu2021encodermi} also makes black-box queries to the shared representation but has access to additional samples from the training data distribution to train shadow models, whereas we assume pure black-box access.
\section{Analyzing Task-Inference} \label{sec:mean_est}

To motivate further exploration of attacks on deep learning models, we present a simplified analog of multitask learning via mean estimation. Proofs of our theorems in this Section can be found in Appendix~\ref{sec:proofs}.\\ 

\noindent Suppose that there exists a set of \emph{task means} $M = \{\mu_1, \dots, \mu_T\}$ which are sampled i.i.d.\ from $\mathcal{N}(\bar{\mu}, \bar{\sigma}^2 \mathbb{I}_d)$ where $\bar{\mu}$ is the "true" mean that we would like to estimate. For each "task" $\mu_i$, we can sample $N$ (where $N$ is relatively small) points i.i.d. from $\mathcal{N}(\mu_i, \sigma^2 \mathbb{I}_d)$ and store them in the set $X_i = \{x_{i, 1}, \dots, x_{i, N} \}$. Since no individual dataset $X_i$ would yield an accurate estimate of $\bar{\mu}$ we can instead compute the sample "multitask" mean using all of the $X_i$'s as 
\[
    \hat{\mu} = \frac{1}{T}\sum_{i=1}^{T} \left( \frac{1}{N} \sum_{j=1}^{N} X_{i,j} \right)
\]

\noindent This "multitask" mean can be viewed as analogous to the weights of the shared representation in multitask learning, because we average over data, $X_i$, sampled from a mixture of Gaussian distributions, $\mathcal{N}(\mu_i, \sigma^2 \mathbb{I}_d)$, which share features, to produce an accurate estimate of the underlying common parameter, $\bar{\mu}$. We note that $\hat{\mu}$ is Gaussian with expectation and covariance

\[
\ex{\mu, X}{\hat{\mu}} = \bar{\mu} \qquad
\cov{}{\hat{\mu}} = \Big( \frac{\bar{\sigma}^2}{T} + \frac{\sigma^2}{N \cdot T} \Big) \cdot \mathbb{I}_d
\]

\begin{figure*}[t!]
    \centering
    \begin{subfigure}[c]{0.40\textwidth}  
        \centering
        \includegraphics[width=\textwidth]{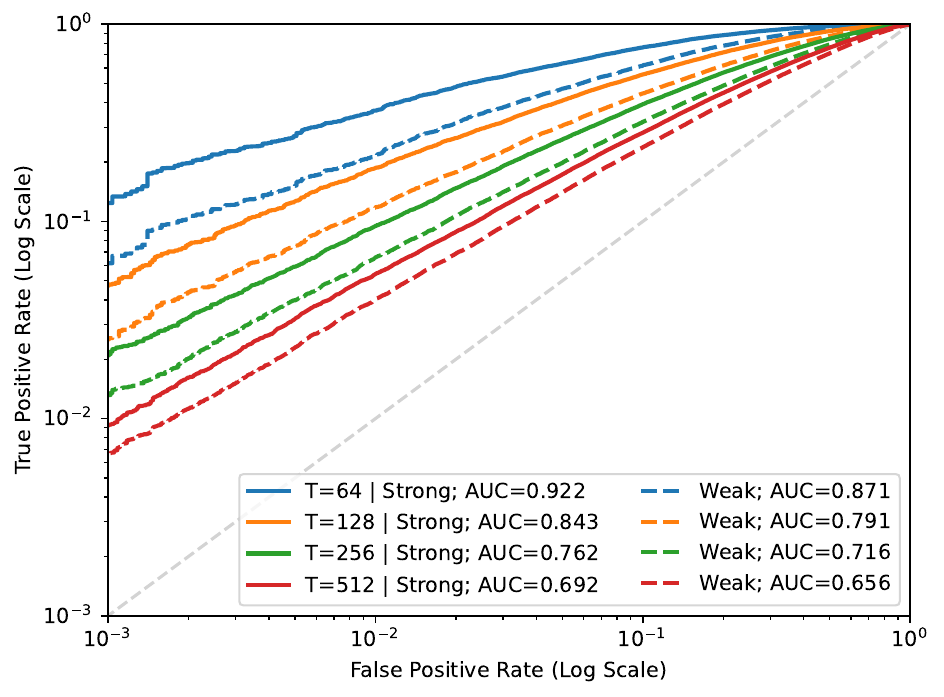}
        \caption{ROC Curves for Strong and Weak Tracing Adversaries}
        \label{fig:attack_roc}
    \end{subfigure}
    \quad \quad \quad \quad \quad
    \begin{subfigure}[c]{0.45\textwidth}  
        \centering
        \includegraphics[width=\textwidth]{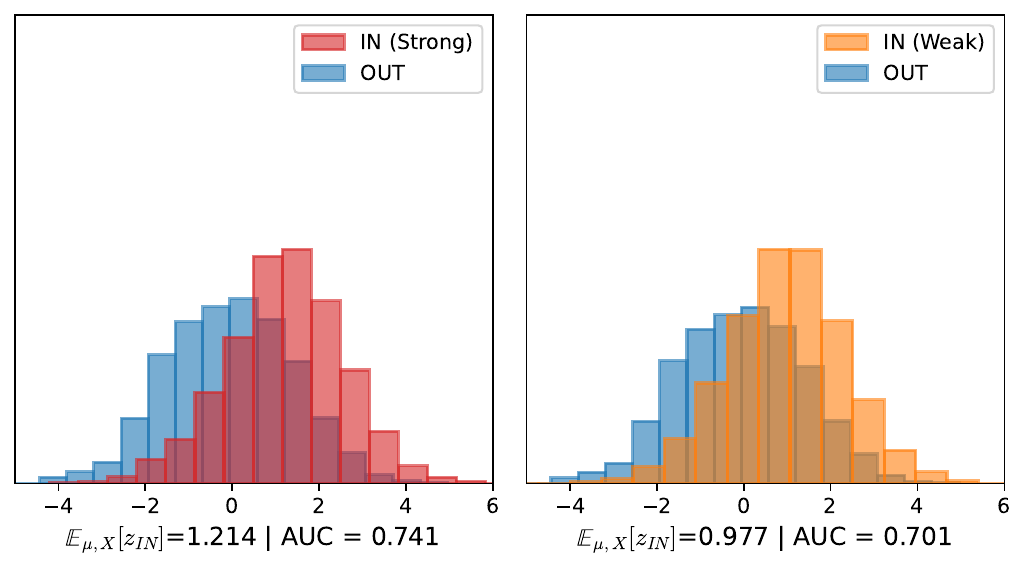}
        \caption{Distribution of Test Statistic $z$}
        \label{fig:test_statistic_mean}
    \end{subfigure}
    \caption{Results of Our Tracing Attack on Multitask Mean Estimation ( $T = 256$; $d=256$; $N=8$; $k=4$ )}
    \label{fig:mean_est_sim}
\end{figure*}

\subsection{Tracing Attack for Task Inference} \label{sec:tracing}

Here, we construct an adversary, based on prior work on tracing~\cite{dwork-2015-trace}, with similar variants to the \textbf{strong} and \textbf{weak} adversaries detailed in Section~\ref{sec:threatmodel}. We consider a challenger who releases the statistic $\hat{\mu}$ and an adversary who wants to learn whether a given task $\mu_i$ was included in the dataset that was used to compute $\hat{\mu}$.

\noindent One possible attack would be the following:

\begin{enumerate}[labelsep=*,leftmargin=15pt]
    \item The adversary receives a challenge set, or batch of data, $B$, from the challenger where $| B | = k \leq N$, such that the samples come from a task that was used to compute $\hat{\mu}$, $\mu_{IN} = \mu_i$, or a task that was \emph{not} used to compute $\hat{\mu}$ but comes from the same underlying task distribution, $\mu_{OUT} \sim \mathcal{N}(\bar{\mu}, \bar{\sigma}^2 \mathbb{I}_d)$
    \item First, the adversary computes the sample mean of the batch
    \[
    \mu_{B} = \frac{1}{k} \sum_{x \in B}{x} 
    \]
    \item Then, the adversary computes the test statistic
    \[
    z = \inn{\hat{\mu} - \Bar{\mu}}{\mu_B - \Bar{\mu}}
    \]
    \item Lastly, the adversary performs a thresholded classification, returning $z > \gamma$ for some threshold $\gamma$
\end{enumerate}

Our adversary adapts a test statistic originally intended for membership-inference~\cite{dwork-2015-trace} that measures the correlation between the released statistic and the samples available to the adversary. However, unlike the membership-inference setting, the adversary is attempting to detect traces of the \textit{task}, or data's distribution, $\mu_i$, rather than any particular sample from one of the datasets $X_i$. In the following theorems, we find that the inclusion of a task, $\mu_i$, can be inferred by both the \textbf{strong} and \textbf{weak} adversaries, and the \textbf{strong} adversary gets a slight advantage from having access to samples in $X_i$.

\begin{thm}[Strong Adversary]  \label{thm:strong_adv}

Let $\tau$ be the index of the target task and suppose that the challenger sends the adversary a challenge set of $k$ samples $B$ such that, when the task is \textbf{IN}, the $k$ samples are drawn uniformly at random from $X_\tau$. Then the expected value of the statistic, $z$, when $\mu_\tau$ is \textbf{OUT} is 
    \[
        \ex{\mu, \: X}{z_{\textit{OUT}}} = 0
    \]
    and when  $\mu_\tau$ is \textbf{IN}, the expected value of $z$ is
    \[
        \ex{\mu, X}{z_{\textit{IN}}} = \frac{d}{T} \Big( \bar{\sigma}^2 + \frac{\sigma^2}{N} \Big)
    \]
\end{thm}


In Theorem~\ref{thm:strong_adv}, we see that the expectation of the \textbf{strong} adversary's test statistic, $z$, primarily grows in the dimension of the data divided by the total number of tasks used to learn the multitask mean, $\hat{\mu}$. Through this result, we demonstrate that detecting traces of tasks can be seen as a generalization of membership inference, as setting the number of tasks $T=1$, sampling $N$ datapoints, $X$, only from the single task to estimate the task mean, and allowing the adversary to use $k=1$ sample reduces $z_{IN}$ to the membership-inference test statistic, which has expectation $\Theta \left( \frac{d}{N} \right)$. In our setting, if the adversary wants to trace a single sample in the multitask mean, the expectation of their statistic would be $\Theta \left( \frac{d}{NT} \right)$. Because the total dataset size for our mean estimation is $T\cdot N$ with $N \ll T$, tracing tasks is an easier objective for the adversary than tracing individual samples, as there a larger separation between the distributions of $z_{IN}$ and $z_{OUT}$. Moreover, because the dataset is composed of multiple task distributions, a \textbf{weak} task-inference adversary can mount attacks by using data that comes from one of the included tasks but was never used in the estimation of $\hat{\mu}$.

\begin{thm}[Weak Adversary] \label{thm:weak_adv}
    Let $\tau$ be the index of the target task and suppose the challenger sends the adversary a challenge set of $k$ samples $B$ such that, when the task is \textbf{IN}, \textbf{the $k$ samples are drawn i.i.d. from the same distribution as $X_\tau$, $\mathcal{N}(\mu_\tau, \sigma^2 \mathbb{I}_d)$.} Then the expected value of the statistic, $z$, when $\mu_\tau$ is \textbf{OUT} is 0. When $\mu_\tau$ is \textbf{IN}, the expected value of $z$ is 
    \[
        \ex{\mu, X}{z_{\textit{IN}}} = \frac{d}{T} \bar{\sigma}^2
    \]
\end{thm}

Informally, Theorem~\ref{thm:weak_adv} shows that the expectation of the weak adversary's test statistic depends only on the number of tasks, and does not benefit from a smaller number of total training samples. Similar to the \textbf{strong} case, when analyzing the \textbf{weak} case, we split the statistically dependent and independent components of the sum, as the batch is drawn from the underlying task distribution for some fixed task which was included in the estimate. In contrast with the \textbf{strong} case, there is no overlap between the adversary's challenge set and the training data. Thus, we need not account for overlapping $X_{\tau,j}$'s. Rather, we only need to consider that the challenge set samples come from the same distribution, or task, as one of the tasks in the dataset used to compute $\hat{\mu}$. \\ 

\begin{thm}[Variance of $z$] \label{thm:variance}
    The variance of $z$ when $\mu_\tau$ is \textbf{OUT} is
    \[
        \var{\mu, X}{z_{\textit{OUT}}} = \frac{d}{T} \Big[ \bar{\sigma}^4 + \frac{\sigma^4}{kN} + \Big( \frac{k+N}{kN} \Big) ( \bar{\sigma}^2\sigma^2 ) \Big]
    \]
    \\
    When $\mu_\tau$ is \textbf{IN}, 
    \[
        \var{\mu,X}{{z}_{\textit{IN}}} \leq {3} \var{\mu,X}{{z}_{\textit{OUT}}}
    \]
\end{thm}

We note that the \textbf{strong} adversary always achieves greater attack success rates than the \textbf{weak} adversary because the distance between the expected value of the test statistic, $z$, when the task is \textbf{IN} or \textbf{OUT} is strictly larger than that of the \textbf{weak} adversary at a similar level of variance. 
Furthermore, this attack on mean estimation shows that the adversary's success is dependent on two distinct parts: the knowledge they receive from having \textit{member} data (i.e. the $\frac{d }{TN} \sigma^2 $ term) and the knowledge they have about the \textit{distributions}, or \textit{tasks}, that compose the dataset (i.e. the $\frac{d}{T} \bar{\sigma}^2$ term). To empirically verify our results, we simulated the attack with parameters $T = 256$, $d=256$, $N=8$, and $k=4$. The ROC curves and distribution of $z$ for both the \textbf{strong} and \textbf{weak} adversaries are shown in Figure~\ref{fig:mean_est_sim}.

\section{Methodology} \label{sec:methodology}

In this section, we realize the threat model in Section~\ref{sec:threatmodel} for machine learning models by introducing two efficient, purely black box attacks on shared representations. 


\subsection{Our Task Inference Attacks}

Here, we describe two simple, purely black-box, efficient attack algorithms that are most applicable to situations where MTL is used to train a model over few samples per task (e.g. \textit{personalization}). Our attack algorithms are inspired by previous work on membership-inference attacks on image encoder models~\cite{liu2021encodermi}, where the key insight is following: For any given image, embeddings produced by well-fit encoder models are robust to augmentations, such as random rotations and horizontal flips. Computing the similarity score of the embeddings of augmentations results in a statistic that has high membership-inference signal. 

In this work, we hypothesize that rich, shared representations learned in MTL and collaborative learning settings produce embeddings where samples from the same task exhibit both positive and negative codependencies, similar to the relationship between augmentations of the same image studied by~\cite{liu2021encodermi} and~\cite{song-2020-embeddingleakage} in the contrastive learning setting. We also believe that, in well-generalized machine learning models, some amount of correlation between samples from the same task distribution in the shared representation space exists regardless of how samples are labeled within the MTL optimization. 
%

\subsubsection{Coordinate-Wise Variance Attack (Algorithm~\ref{alg:var_attack})} To mount the \textit{coordinate-wise variance attack}, the adversary first queries the shared representation on a batch of $k$ data samples from a given task, then aggregates the embeddings into a set $E = \{ h_\theta(x_1), \dots, h_\theta(x_k) \}$. Then, the adversary computes the empirical covariance matrix of $E$ and takes the trace divided by the dimension of the embedding vectors to be the task-inference statistic, $z$. This is equivalent to computing the sum of coordinate-wise variances of the embeddings. Lastly, the adversary sets a threshold $\gamma$ such that any task with $z > \gamma$ is labeled as $IN$, and any task with $z < \gamma$ is labeled as $OUT$.

\begin{algorithm}[H]
    \caption{Coordinate-Wise Variance Attack}
    \label{alg:var_attack}
    \begin{algorithmic}[1]
       \State {\bfseries Input:} A shared representation $h_\theta$, a challenge set  $X_\tau$
       \State $E = \{ \}$ 
       \For{$x_{i} \in X_\tau$}
           \State $e_i \gets h_\theta(x_i)$ \Comment{Query representation on challenge set}
           \State $E \gets E \cup \{ e_i \}$
       \EndFor
       \State Compute $Q$, the empirical covariance matrix of $E$
       \State \Return $tr(Q)$ \Comment{Return avg.\ coordinate-wise variance}
    \end{algorithmic}
\end{algorithm}

\subsubsection{Pairwise Inner Product Attack (Algorithm~\ref{alg:inner_attack})} For our second attack, we once again use the fact that the shared representation produces close embeddings for data from the same task, but we measure similarity of the entire embedding vectors rather their individual coordinates by taking their \textit{inner products}. As in the previous attack, the adversary first queries the shared representation and aggregates the embeddings $E$. Then, for each unique pair of data samples $(x_i, x_j) \: i \neq j$, the adversary computes the absolute value of the inner product of their embeddings, $z$, and stores the value in a set $S$. The adversary then takes the mean of the $z$'s ($\Bar{S}$) and applies a threshold $\gamma$ as in the variance attack. Alternatively, the adversary can compute the normalized inner products, or \emph{cosine similarities}, of each pair of embeddings. In our evaluation, cosine similarity yields better attack performance on language models, while the standard inner product attack achieves higher true positive rates on vision models.

\begin{algorithm}[H]
    \caption{Pairwise Inner Product Attack}
    \label{alg:inner_attack}
    \begin{algorithmic}[1]
       \State {\bfseries Input:} A shared representation $h_\theta$, a challenge set  $X_\tau$
       \State $E = \{ \}$ 
       \For{$x_{i} \in X_\tau$}
           \State $e_i \gets h_\theta(x_i)$ \Comment{Query representation on challenge set}
           \State $e_i 
           \gets W (e_i - \bar{e}) $ \Comment{Apply whitening}
           \State $E \gets E \cup \{ e_i \}$ \Comment{Store normalized embedding}
       \EndFor
       \State $S = \{ \}$
       \For{all unique pairs $(e_i, \: e_j); \: i \neq j; \: e_i, e_j \in E$}
           \If{\texttt{use\_cosine\_similarity}}
               \State $(e_i, e_j) \gets \left( \frac{e_i}{\| e_i \|}, 
               \: \frac{e_j}{\| e_j \|} \right)$
           \EndIf
           \State $z \gets \vert \inn{e_i}{e_j} \vert$
           \State $S \gets S \cup \{ z \}$
       \EndFor
       \State \Return $\Bar{S}$ \Comment{Return average pairwise inner product}
    \end{algorithmic}
\end{algorithm}

\subsubsection{Normalizing Embeddings} Unlike prior work on membership-inference~\cite{shokri-2016-mi,carlini-2022-lira} and property inference~\cite{chaudhari-2022-snap, mahloujifar2022propertypoisoning} attacks that train shadow models, our task-inference adversary requires only query access to the shared representation, eliminating the need for shadow, or reference, models to calibrate the attack. While a slightly stronger adversary could construct a labeled, auxiliary dataset of known \textbf{OUT} tasks to calibrate their attack and perform a more powerful exact one-sided test, as in~\cite{carlini-2022-lira}, we find that simple thresholding is sufficient to achieve high success rates while maintaining the purely black-box aspect of the threat model. Thus, using the adversary's sampling access to task data and query access to the shared representation, we can attempt to reduce the noise in the embedding vectors by using a \textit{whitening transformation}.

Applying a whitening transformation to the embeddings transforms them into random vectors with covariance equal to the identity matrix. By doing this, we attempt to contract the axes in the the representation space that dominate our task-inference statistics. We find that whitened embeddings often provide better signal to the adversary for our inner product attack (Algorithm~\ref{alg:inner_attack}) than raw embeddings. To compute the transformation, we first estimate the covariance matrix of the embedding space by pooling all of the embeddings available to the adversary, regardless of task or inclusion in the model's training set, and computing the regularized covariance matrix $\Sigma_{reg} = (Q + \lambda \cdot \frac{tr(Q)}{d} \mathbb{I}_d)$ where $Q$ is the sample covariance matrix, $d$ is the embedding dimension, and $\lambda$ is a small constant. Once a well-conditioned $\Sigma_{reg}$ has been estimated, a common choice for the whitening transformation would be $W = \Sigma_{reg}^{-1/2}$. To ensure that the covariance estimate is not dependent on a challenge task's data, we compute whitening transformations for each of the $2T$ tasks that we input to our attack, leaving out the data from one task each time. When applying $W$ to the embeddings, we additionally center the embeddings by computing $\bar{e}$, the sample mean of \textit{all} embedding vectors. Thus, we normalize a given embedding $e$ by computing $W(e - \bar{e})$.
\section{Evaluation} \label{sec:eval}

In this section we present a thorough evaluation of our task-inference attacks on MTL models trained for differing use cases across both vision and language domains. We showcase our attacks' success and study their performance when the adversary has access to training samples (i.e. \textbf{strong}) and when they only have access to fresh samples from the challenge task (i.e. \textbf{weak}). Additionally, we attempt to understand the factors that lead to task-inference leakage by observing how task-inference changes as a function of the MTL model's generalization loss. Additional experiments and figures can be found in Appendix~\ref{appdx:eval}. The results we present are underscored by our attacks' efficiency, taking as few as four samples per task and requiring \emph{only black-box query access} to the shared representation. In our experiments, this translates to each black-box query and task-inference prediction taking roughly \emph{one tenth of a second} on a single RTX 4090 GPU.

\subsection{MTL Training} 
In this study, we consider the original instantiation of MTL training which is described in the seminal work on the topic of MTL~\cite{caruana-1997-multitask}. We highlight that this MTL algorithm is a slight variation of centralized FedSGD~\cite{mcmahan2017FLoriginal}, one of the baseline algorithms in the original paper on collaborative learning. For all of our models, we allow all tasks to share all but the task-specific classification heads. In our experiments, the shared representation takes the form of a neural network, and the classification heads are linear. During MTL training, we perform a forward pass and route each embedding vector to its corresponding task-specific layer. Thus, in the backward pass, the loss from \textit{all} tasks is used to train the shared encoder, while only the loss from each particular task is used to train each linear layer. Algorithm~\ref{alg:mtl_training} shows the general training loop we use for MTL. In practice, we use AdamW~\cite{loshchilov-2017-adamw} for our \texttt{Update} step, and we perform shuffling over tasks and batching for efficiency. Additionally, when there is sufficient data available per task, we randomly subsample data from each task for each training epoch. 

\begin{algorithm}[h]
    \caption{Multitask Learning Training Loop}
    \label{alg:mtl_training}
    \begin{algorithmic}[1]
       \State {\bfseries Input:} A shared representation $h_\theta$, $T$ task-specific layers $\{g_{\beta_1}, \dots, g_{\beta_T}\}$, a training set split by task $D~=~\{(X_{1}, y_1 ) \dots (X_{T}, y_T )\}$, number of training rounds $t$, a loss function $\mathcal{L}$
       \For{$i \in [t]$}
            \State $\ell_{\textit{MTL}}$ $\gets 0$ \Comment{Initialize multitask loss}
            \For{$j \in [T]$} 
                \State $p_j \gets g_{\beta_j} ( \: h_{\theta}(\: X_j \:) \:)$ \Comment{Get predictions for task $j$}
                \State $\ell_{\textit{MTL}} \gets \ell_{\textit{MTL}} + \frac{1}{T} \mathcal{L}(p_{j}, y_j)$
            \EndFor
            \For{$j \in [T]$} \Comment{Update each task specific layer}
                \State $\beta_j \gets \texttt{Update} \left( \beta_j, \: \nabla_{\beta_j} \ell \right)$ 
            \EndFor
            \State $\theta \gets \texttt{Update} \left(\theta, \:  \nabla_{\theta} \ell \right)$ \Comment{Update shared representation}
        \EndFor
        \Return $h_\theta$
    \end{algorithmic}
\end{algorithm}

Since the number of samples per task can be smaller than the embedding dimension, we use several regularization techniques (detailed in Appendix~\ref{appdx:training}) to keep the task-specific linear layers from overfitting to noisy embeddings early during MTL training. To summarize, we use large values for weight decay, or $L_{2}$ regularization, learn a bottleneck on the shared representation to decrease the embedding dimension, normalize gradients~\cite{chen-2018-gradnormalization} across tasks during model updates, apply clipping to gradients, and, in our vision experiments, apply standard augmentations to training images.

\subsection{Models and Datasets}

\subsubsection{Vision Models} In our vision experiments, we use ResNet models~\cite{ResNet} of differing sizes as the shared representation for MTL. The ResNet architecture has gained widespread adoption across computer vision applications due to its computational efficiency and high performance on a variety of datasets. This model architecture makes use of residual connections, which stabilize training and convergence to help maintain high utility when trained on large image datasets. When available, such as in our CelebA experiments, we use larger ResNet models which are pretrained on the ImageNet~\cite{imagenet} dataset from the PyTorch~\cite{paszke2019pytorch} library.

\subsubsection{Language Models} For our language experiments, we use the BERT~\cite{devlin2018bert}  architecture as the shared representation. In particular, we use the downsized variants of BERT~\cite{bertsmall_1, bertsmall_2} which are pretrained on MNLI~\cite{MNLI} for downstream sequence classification.


\subsubsection{Vision Datasets} We evaluate our task-inference attacks on both the CelebA~\cite{liu-2015-celeba} faces dataset and the Federated EMNIST (FEMNIST)~\cite{caldas-2019-leafbenchmark} handwritten character and images dataset. CelebA contains high-resolution images of celebrities, each with a unique identifier, and 40 facial attributes with binary labels. The FEMNIST dataset contains 28x28 grayscale images of 62 different types of handwritten characters. In our experiments on CelebA, we split the dataset into tasks by person and by facial attribute for each of the two MTL use cases we consider. We split FEMNIST by writer and train an MTL model for character recognition, personalized to each writer's hand-drawn images.



\subsubsection{Language Datasets} To evaluate our attack in the language setting, we use the Stack Overflow~\cite{stackoverflow} dataset, which consists of posts from the Stack Overflow website with corresponding ratings, user ID, and topic tags for each post. In our evaluation, we split the posts into tasks by user ID and by topic for each of the MTL use cases. 



\subsection{Metrics}

To measure the performance of our task-inference attacks, we use metrics that are commonly found in the inference attack literature~\cite{carlini-2022-lira, zarifzadeh2024lowcostmia, ye-2022-enhanced_mi, kandpal2023userinference}. In particular, we analyze the ROC curves of our attack, which measure the relationship between true positive rate (TPR) and false positive rate (FPR). As summary statistics, we report the area under the ROC curve, or AUC, as well as the TPR at fixed low FPR (e.g. 1\%). Our evaluation distinguishes between two adversarial settings: a \textbf{strong} adversary with access to training samples and a \textbf{weak} adversary limited to auxiliary data from the challenge task that was never seen during training. Thus, we present ROC curves for each adversary. Additionally, we highlight that our attacks can be performed in a purely black-box setting. This results in the adversary not necessarily having sufficient knowledge to select the threshold that yields an optimal tradeoff between TPR and FPR. So, we also report the (TPR, FPR) pairs when thresholding our test statistics at the 50th, 75th, and 90th percentiles.


\begin{figure}[t]
    \centering
    \begin{subfigure}[b]{0.32\columnwidth}
        \centering
        \includegraphics[width=\textwidth]{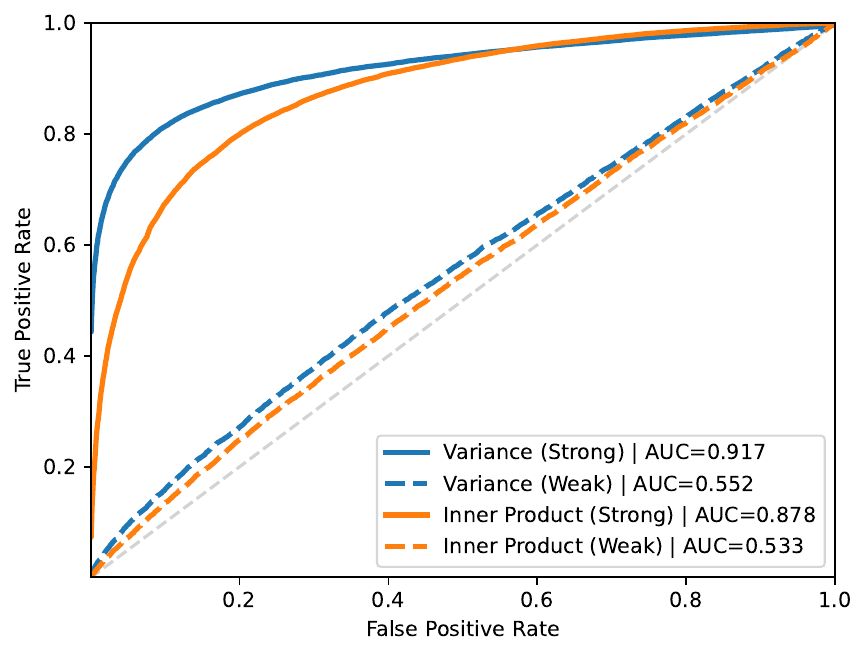}
        \caption{CelebA}
        \label{fig:celeba}
    \end{subfigure}
    \hfill
    \begin{subfigure}[b]{0.32\columnwidth}
        \centering
        \includegraphics[width=\textwidth]{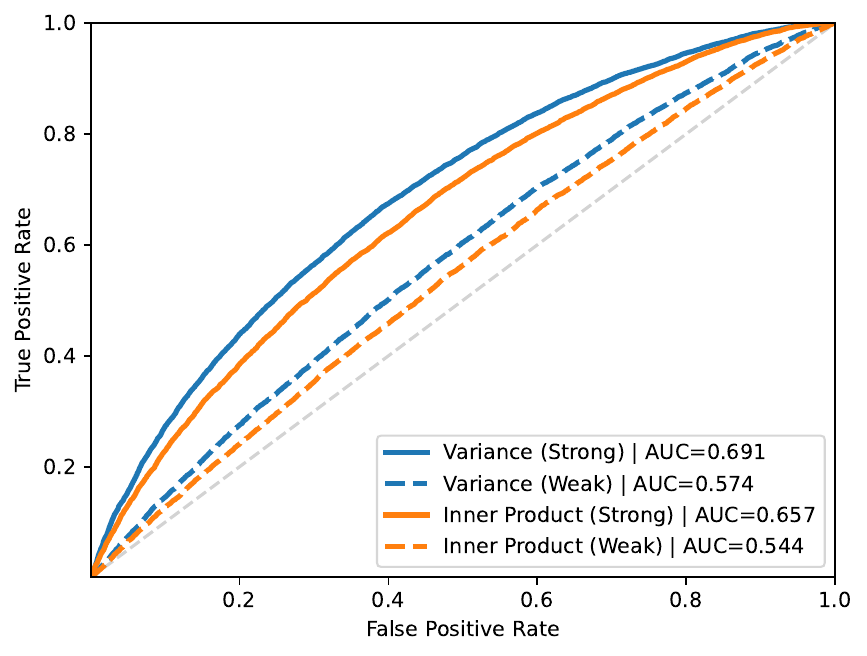}
        \caption{FEMNIST}
        \label{fig:femnist}
    \end{subfigure}
    \hfill
    \begin{subfigure}[b]{0.34\columnwidth}
        \centering
        \includegraphics[width=\textwidth]{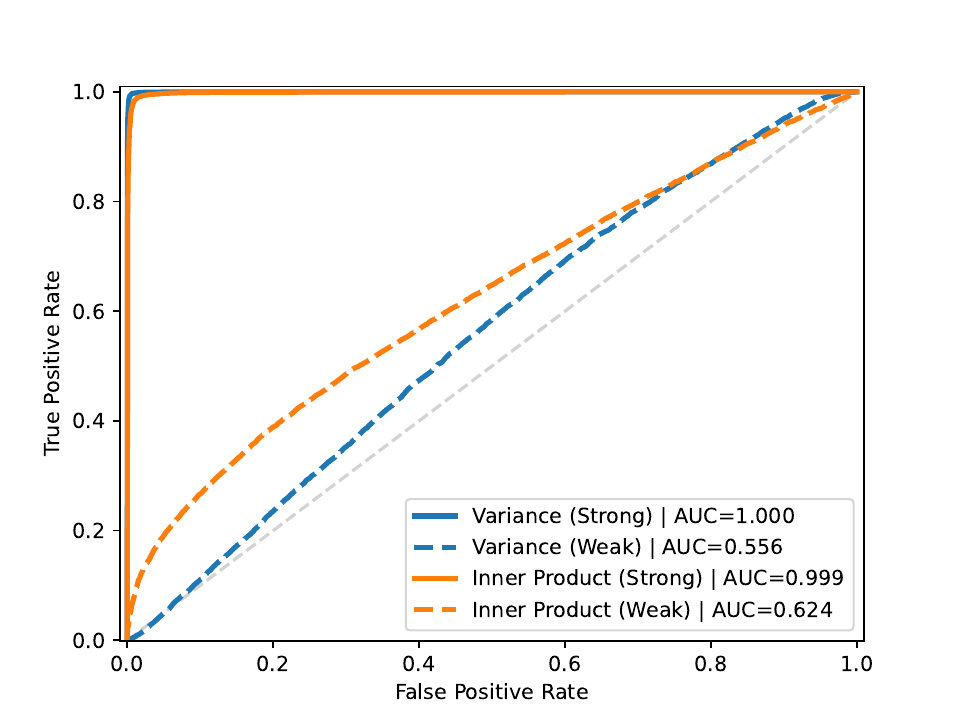}
        \caption{Stack Overflow}
        \label{fig:user_level_lm}
    \end{subfigure}
    \caption{ROC of Task-Inference Attacks (Personalization)}
    \label{fig:personalization_roc}
\end{figure}

\subsection{Personalization} \label{sec:personalization}

Here, we present the results of our experiments in the vision and language settings where MTL is used for \textit{personalization}. Across all experiments, the datasets are organized by individual users, with each user contributing multiple samples. In the MTL framework, the multiple tasks are the users, and each user's task-specific layer, which is linear in this case, receives updates from their data. In a distributed learning setting, this could be thought of as locally adapting the model to user-specific data. The shared representation is trained on all of the tasks. When the adversary infers a task's inclusion in multitask training, they are leaking whether or not an individual's data was present at all in updates to the shared representation. Like in membership-inference, the \textbf{strong} adversary will have a batch of real training samples at their disposal to perform this test, but the \textbf{weak} adversary receives fresh samples that belong to the user and were never seen during training.


\begin{center}
    \begin{table}
    \centering
  \caption{(TPR, FPR) and Balanced Accuracy for Black-Box Percentile Thresholds for Inner Product Attack at $50^{th}$, $75^{th}$, and $90^{th}$ Percentile Thresholds~(Personalization)}
  \label{tab:percentiles_personalization}
  \resizebox{0.9\linewidth}{!}{%
        \begin{tabular}{ll cc cc cc}
            \toprule
            & & \multicolumn{2}{c}{$50^{th}$ Percentile} & \multicolumn{2}{c}{$75^{th}$ Percentile} & \multicolumn{2}{c}{$90^{th}$ Percentile} \\
            Dataset & Access & (TPR, FPR) & Acc. & (TPR, FPR) & Acc. & (TPR, FPR) & Acc. \\
            \midrule
            \multirow{2}{*}{CelebA (Fig~\ref{fig:celeba})} 
            & \textbf{Strong} & $(80\%, 20\%)$ & $80\%$ & $(46.7\%, 3.3\%)$ & $71.7\%$ & $(19.5\%, 0.5\%)$ & $59.5\%$ \\
            & \textbf{Weak} &$(52.4\%, 47.6\%)$ & $52.4\%$ & $(27.4\%, 22.5\%)$ & $52.5\%$ & $(11.5\%, 8.5\%)$ & $51.6\%$ \\
            \midrule
            \multirow{2}{*}{FEMNIST (Fig~\ref{fig:femnist})} 
            & \textbf{Strong} & $(61.1\%, 38.8\%)$ & $61.1\%$ & $(33.5\%, 16.4\%)$ & $58.7\%$ & $(14.3\%, 5.7\%)$ & $54.3\%$\\
            & \textbf{Weak} &$(53.2\%, 46.8\%)$ & $53.2\%$ & $(27.2\%, 22.8\%)$ & $52.2\%,$ & $(11.2\%, 8.8\%)$ & $51.2\%$ \\
            \midrule
            \multirow{2}{*}{Stack Overflow (Fig~\ref{fig:user_level_lm})} 
            & \textbf{Strong} & $(98.7\%, 1.2\%)$ & $98.7\%$ & $(50.0\%, 0\%)$ & $75.0\%$ & $(19.9\%, 0\%)$ & $59.9\%$ \\
            & \textbf{Weak} & $(58.2\%, 41.7\%)$ & $58.2\%$ & $(34.1\%, 15.8\%)$ & $59.1\%$ & $(16.3\%, 3.6\%)$ & $56.3\%$ \\
            \bottomrule
        \end{tabular}
    }
    \end{table}
\end{center}





\subsubsection{CelebA} \label{sec:celeba}

First, we report the results for our attacks in the vision setting. We train a ResNet34~\cite{ResNet} from an ImageNet~\cite{imagenet} checkpoint with one linear layer per task on the CelebA~\cite{liu-2015-celeba} dataset to perform binary facial attribute detection, with each task corresponding to a person in the dataset. We filter CelebA such that each unique individual has roughly 30 images containing their face, and we hold out 8 samples per individual to mount an attack with the \textbf{weak} adversary. In our experiments, we jointly train the MTL model on 256 total tasks with 22 samples per task, or roughly 5600 total samples, and thus use several regularization techniques to ensure that the ResNet layers learn performant embeddings. During each training step, we randomly sample a batch of tasks and update their corresponding task heads, while aggregating these tasks' gradients to update the shared representation.

Once the model is trained, we run 128 trials of each attack on all 512 tasks (256 $IN$ and 256 $OUT$), using 8 samples per task for the \textbf{strong} adversary and 4 of the held out samples per task for the \textbf{weak} adversary. The results of this experiment are shown in Figure~\ref{fig:celeba} and Table~\ref{tab:percentiles_personalization}. We find that while both the \textbf{strong} and \textbf{weak} adversaries can achieve non-trivial success rates in determining a task's inclusion in training, the true positive rate of the \textbf{weak} adversary is bounded above by the true positive rate of the \textbf{strong} adversary for any fixed FPR. Moreover, we see that our variance attack nets better TPR in the low FPR regime, but the inner product attack sees better FPR at higher FPR. At a fixed FPR of $1\%$, the TPR of our variance attack on CelebA achieves a TPR of $61.2\%$  and  $2.9\%$ for the for the \textbf{strong} and \textbf{weak} adversaries, respectively.

\subsubsection{FEMNIST} \label{sec:femnist}

We report additional results in the vision setting for Federated EMNIST~\cite{caldas-2019-leafbenchmark}. We train a MTL model with ResNet8~\cite{ResNet} as the shared representation and personalize each task-specific layer to a unique writer in the dataset. Each writer in the FEMNIST has much more data than each individual in CelebA. Thus, we train the MTL model on 128 tasks with 128 samples per task and hold out 16 samples to mount our attack with the \textbf{weak} adversary. For each of the 128 runs of the attack, the \textbf{strong} adversary receives 16 training samples, and the \textbf{weak} adversary receives 8 of the held out samples. Figure~\ref{fig:femnist} and Table~\ref{tab:percentiles_personalization} show the ROC curves and (TPR, FPR) pairs of our task-inference attacks on Federated EMNIST. We observe that both the \textbf{strong} and \textbf{weak} adversaries are able to achieve nontrivial success rates, and the \textbf{strong} adversary sees a TPR of 3\% at a fixed 1\% FPR. When the \textbf{weak} adversary picks the $90^{th}$ percentile threshold, the inner product attack yields a 2.5$\times$ higher TPR than FPR.


\subsubsection{Stack Overflow} \label{sec:stackoverflow}

Next, we present the results for our language experiments on the StackOverflow posts dataset. We train a BERT Small~\cite{bertsmall_1} (29M parameters) to perform topic classification, where the tasks are users who contributed posts to Stack Overflow. We filter the dataset such that each individual has at least 48 total posts; 32 posts for the model's training set and 16 to be used by the \textbf{weak} adversary. Then, we split the dataset into 128 \textbf{IN} tasks and 128 \textbf{OUT} tasks, yielding two datasets of roughly 10k posts, with each user contributing about 140 posts on average. We train the shared representation and task-specific linear layers jointly on the \textbf{IN} dataset. The MTL model is trained to detect the presence of 256 unique topics in posts, and each post can have multiple corresponding topics. 

After training is complete, we run each attack on all 256 tasks for 128 trials, using 8 samples and 4 samples per trial for the \textbf{strong} and \textbf{weak} adversary, respectively. Figure~\ref{sec:stackoverflow} and Table~\ref{tab:percentiles_personalization} show the results for this attack. We find that the \textbf{strong} adversary achieves nearly perfect AUC for both attacks and FPR of 0\% at the $75^{th}$ and $90^{th}$ percentile thresholds over all runs of the attack. The \textbf{weak} adversary again sees nontrivial AUC for both attacks, and we see a tradeoff between the TPR of the variance and inner product attacks at particularly high FPR. At fixed FPR of 1\% the inner product adversary with \textbf{weak} access to the data achieves a TPR of 8.2\%, while the \textbf{strong} adversary achieves a TPR of 98.5\%. By analyzing the topic frequencies in the dataset, we empirically find that individual users tend to only write about a small subset of topics. Across all posts, the median user posts about 31 (or 12.1\%) of the 256 total topics, which could lead to high distinguishability of users in representation space.


\subsubsection{Variance Attack Test Statistic} Across all of our experiments in both MTL settings, the statistic produced by our inner product attack (Algorithm~\ref{alg:inner_attack}) is consistent with respect to the ordering of \textbf{IN} and \textbf{OUT} distributions. In contrast, when studying leakage in MTL for personalization, we observe a discrepancy in the test statistic distributions for vision and language datasets. Figure~\ref{fig:hist_difference} shows that the coordinate-wise variance statistic (Algorithm~\ref{alg:var_attack}) is larger for \textbf{IN} tasks than \textbf{OUT} tasks. In Appendix~\ref{appdx:variance_attack}, we find that MTL training of ResNet models produces task heads which are positively correlated, while our BERT models produce nearly orthogonal task-specific layers. Because the task heads are correlated, the primary signal for our attack is the shared representation's "overconfidence" in certain directions of the embedding space, which yields higher coordinate-wise variance. We find that the task heads of our FEMNIST models have the highest correlation, followed by our CelebA models, then Stack Overflow, which maps to the ordering of the AUC scores we see in our experiments on MTL for personalization.

\begin{figure}[h]
    \centering
    \begin{subfigure}[h]{0.35\columnwidth}
        \centering
        \includegraphics[width=\textwidth]{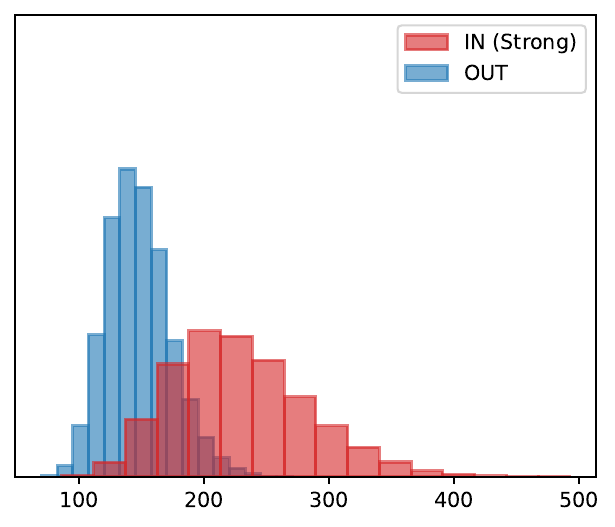}
        \caption{CelebA}
        \label{fig:celeba_hist}
    \end{subfigure}
    \quad \quad
    \begin{subfigure}[h]{0.35\columnwidth}
        \centering
        \includegraphics[width=\textwidth]{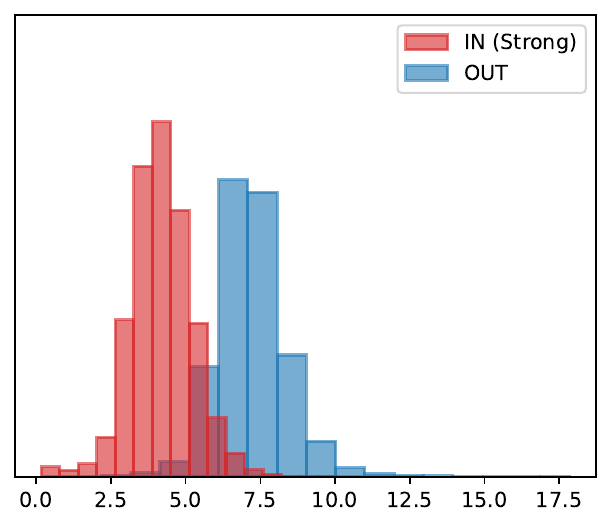}
        \caption{Stack Overflow}
        \label{fig:so_hist}
    \end{subfigure}
    \caption{Distribution of Coordinate-Wise Variance Statistic}
    \label{fig:hist_difference}
\end{figure}


\begin{table*}[t]
    \centering
    \caption{(TPR, FPR) and Balanced Accuracy for Black-Box Percentile Thresholds for Inner Product Attack at $50^{th}$, $75^{th}$, and $90^{th}$ Percentile Thresholds~(Multiple Learning Problems)}
    \label{tab:percentiles_task_level}
    \resizebox{0.88\linewidth}{!}{
        \begin{tabular}{ll cc cc cc}
            \toprule
            & & \multicolumn{2}{c}{$50^{th}$ Percentile} & \multicolumn{2}{c}{$75^{th}$ Percentile} & \multicolumn{2}{c}{$90^{th}$ Percentile} \\
            Dataset & Access & (TPR, FPR) & Acc. & (TPR, FPR) & Acc. & (TPR, FPR) & Acc. \\
            \midrule
            \multirow{2}{*}{CelebA (Fig~\ref{fig:task_level_celeba})} 
              & \textbf{Strong} & $(68.9\%, 31.1\%)$  & $68.9\%$ & $(42.6\%, 7.4\%)$ & $67.6\%$ & $(22.8\%, 0\%)$ & $61.4\%$ \\
              & \textbf{Weak}   & $(52.5\%, 47.5\%)$ & $52.5\%$ & $(29.0\%, 21.0\%)$ &  $54\%$ & $(13.5\%, 6.6\%)$  & $53.5\%$ \\
            \midrule
            \multirow{2}{*}{Stack Overflow (Fig~\ref{fig:task_level_lm})} 
              & \textbf{Strong} & $(87\%, 13.9\%)$ & $87\%$ & $(48.3\%, 1.7\%)$ & $73.3\%$ & $(19.9\%, 0.1\%)$ & $59.9\%$ \\
              & \textbf{Weak}   & $(85.9\%, 14.1\%)$ & $85.9\%$ & $(47.6\%, 2.4\%)$ & $72.6\%$ & $(19.6\%, 0.2\%)$ & $59.7\%$ \\
            \bottomrule
        \end{tabular}
    }
\end{table*}



\subsection{Multiple Learning Problems} \label{sec:multiple_obj}

We present the results of our experiments on models trained using MTL on multiple related learning problems. In contrast to the datasets used to train MTL models in the experiments presented from Section~\ref{sec:personalization}, the datasets are split by learnable classes in the dataset rather than by person. For example, Stack Overflow can be split into its constituent topics, and we can have an MTL model with a tailored head for detecting each individual topic's presence. We draw attention to the fact that in this setting, the labels for the learning problem are directly tied to the task. In other words, if a task is not included, there are no positive labels corresponding to that task in the training dataset. In our experiments on personalized multitask models, while there is some dependency between the task and the labeling of the data (e.g. a celebrity in CelebA with blonde hair is classified as "Blonde" 100\% of the time), there is not a direct mapping from the data's labeling to the tasks in MTL.

\subsubsection{CelebA} \label{sec:celeba_task_level}

For our evaluation on CelebA, we split the dataset into potentially overlapping tasks by facial attribute, 20 \textbf{IN} and 20 \textbf{OUT}, with at least 1024 samples with corresponding positive and negative labelings for each task. Because some tasks have very low positive label frequencies (e.g. < 1\%), this minimum sample size ensures sufficient positive examples for the task head to learn from. We use the ResNet50~\cite{ResNet} architecture as the backbone for our MTL model as there is ample data. Each task has a corresponding linear head for binary attribute prediction, and no two tasks share the same labeling. For example, the task head dedicated to hair color does not make predictions for the task head dedicated to detecting glasses. 

We average the results of our experiments over 8 MTL runs, and run the attacks on the 40 total tasks 128 times each, using 16 samples and 8 samples for the \textbf{strong} and \textbf{weak} adversary, respectively. In Table~\ref{tab:percentiles_task_level} and Figure~\ref{fig:task_level_celeba}, we see that the \textbf{strong} adversary is able to achieve an AUC of $0.745$ using our coordinate-wise variance attack and a TPR of 22.8\% at an empirical FPR of 0\% over all runs when using the $90^{th}$ percentile threshold. The \textbf{weak} sees nontrivial success rates, with a TPR roughly 2$\times$ larger than the FPR when using the $90^{th}$ percentile threshold.

\subsubsection{Stack Overflow} \label{sec:stack_overflow_task_level}

We also run our attack on language models where MTL is used for multiple related learning problems. Unlike the experiments in Section~\ref{sec:stackoverflow}, the posts are not split by user. Instead, each task corresponds to the inclusion of a positively labeled topic in the dataset. For example, if the Stack Overflow topic "Python" is \textbf{IN}, there is a task head dedicated to detecting whether a post includes the topic "Python". If the topic is \emph{not} included, there is no task head that can learn positively labeled "Python" samples. Additionally, because the posts in the Stack Overflow dataset can contain multiple topics, the training data is not disjoint between tasks; the labeling of each task dataset is unique. We randomly split the data into 70 \textbf{IN} tasks and 70 \textbf{OUT} tasks, with each task-specific dataset containing 1024 posts. Using a BERT Medium~\cite{bertsmall_1} (41M parameters) model as the shared representation, we use MTL to learn linear task heads that are specialized to detecting the presence of a single topic within a post.

We run both of our attacks on the BERT model 128 times, using 32 samples and 16 samples each trial for the \textbf{strong} and \textbf{weak} adversaries, respectively. Figure~\ref{fig:task_level_lm} shows that in the setting where tasks correspond directly to the training data labeling in MTL, the gap in attack performance between the \textbf{strong} and \textbf{weak} adversaries becomes small. This can be attributed to the fact that the model necessarily has to achieve high utility on the training tasks in order to solve the learning problems. This contrasts our findings in Section~\ref{sec:personalization}, where the task heads in MTL are not solving distinct learning problems, and the separation between both adversaries is notably larger. In Table~\ref{tab:percentiles_task_level}, we see that the (TPR, FPR) pairs for both adversaries are nearly equal, and the \textbf{weak} adversary is able to achieve a 19.6\% TPR at 0.2\% FPR. 


\begin{figure}[ht]
    \centering
    \begin{subfigure}[b]{0.44\textwidth}
        \includegraphics[width=\textwidth]{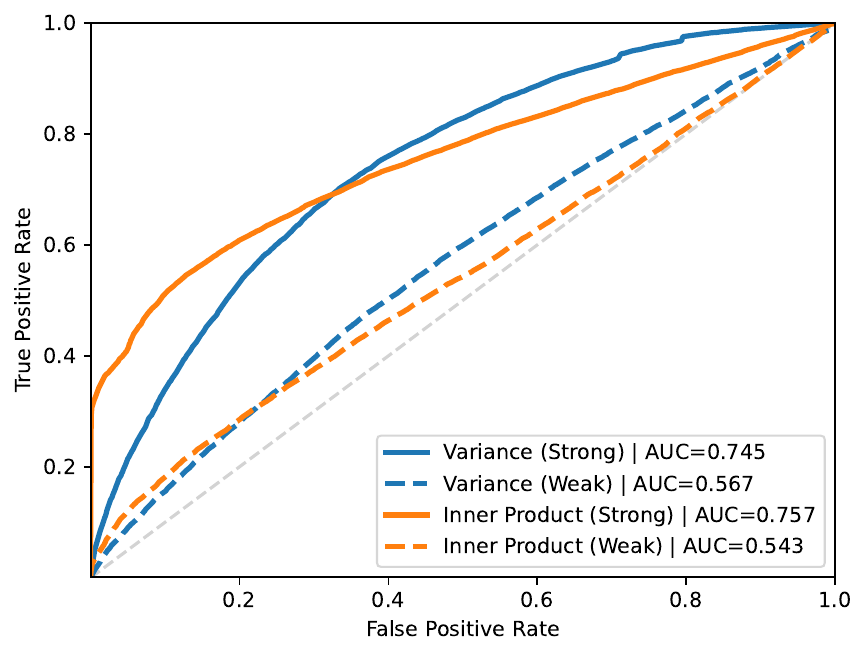}
        \caption{CelebA}
        \label{fig:task_level_celeba}
    \end{subfigure}
    \hfill
    \begin{subfigure}[b]{0.44\textwidth}
        \includegraphics[width=\textwidth]{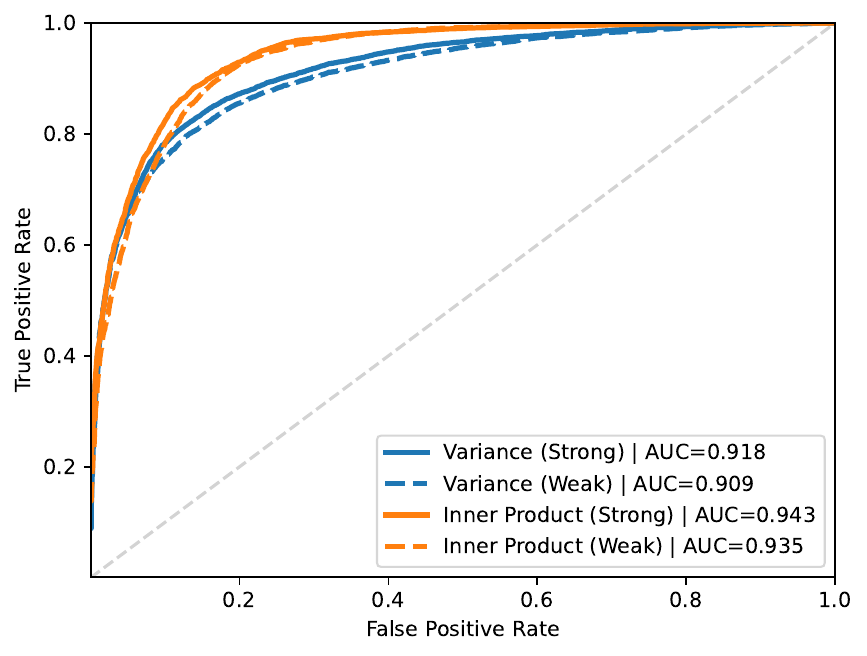}
        \caption{Stack Overflow}
        \label{fig:task_level_lm}
    \end{subfigure}    
    \caption{ROC of Task-Inference Attacks (Multiple Learning Problems)}
    \label{fig:task_inference}
\end{figure}

\subsection{Investigating Sources of Task-Inference Leakage}  Now, we present an investigation of the factors that lead to task-inference. We study how task-inference success varies as a function of gaps between the MTL model's loss on training samples, out of sample data, and out of distribution data. In Appendix~\ref{appdx:synth}, we conduct ablation studies over MTL parameters, such as the embedding dimension and number of tasks, using a synthetic dataset designed to necessitate MTL.

\subsubsection{Task-Inference and Generalization}

The \textbf{weak} adversary in our analysis of tracing attacks for task-inference (Theorem~\ref{thm:weak_adv}) has a notable disadvantage compared to the \textbf{strong adversary}, which has access to actual training samples. When estimating a mean of Gaussian vectors, having several tasks with sparse data yields an accurate enough estimate of the true mean such that the gap between the strong and weak adversaries is observable, but not too large. In our MTL experiments, we see that this gap in AUC between the \textbf{strong} and \textbf{weak} adversaries is larger, especially when each individual task has very few samples, such as in CelebA. In membership-inference, the adversary's success rate is often tied to the model's generalization gap~\cite{2018-yeom-mioverfitting} (and in very overparameterized models, the train loss as it approaches 0) as the model's outputs begin to heavily depend on training examples. To understand the relationship between generalization error and task-inference success, we train Stack Overflow topic classification models, using MTL for personalization over 128 users as in Section~\ref{sec:personalization}, and measure the AUC each adversary achieves using the inner product attack. In total, we train 8 MTL models on random splits of 128 tasks and run the attack 128 times per task, using 8 and 4 samples per run for the \textbf{strong} and \textbf{weak} adversaries, respectively.

Figure~\ref{fig:gengap_pers} shows the \textbf{strong} adversary's AUC as a function of the gap between the training loss and "zero-shot" loss, which we compute by taking the average prediction of 16 random task heads to make predictions on \textbf{OUT} task data. Each point on the scatterplot is shaded corresponding to how many epochs the model had been trained up to that point. We observe a similar trend to membership-inference, where our \textbf{strong} task-inference adversary, which has access to training samples, sees a sharp increase in attack AUC for small generalization gaps. Over many runs of model training on different task splits, we see that the points on the plot tend to concentrate as the generalization gap grows.

\begin{figure*}[t]
    \centering
    \begin{subfigure}[b]{0.57\linewidth}
        \centering
        \includegraphics[width=\textwidth]{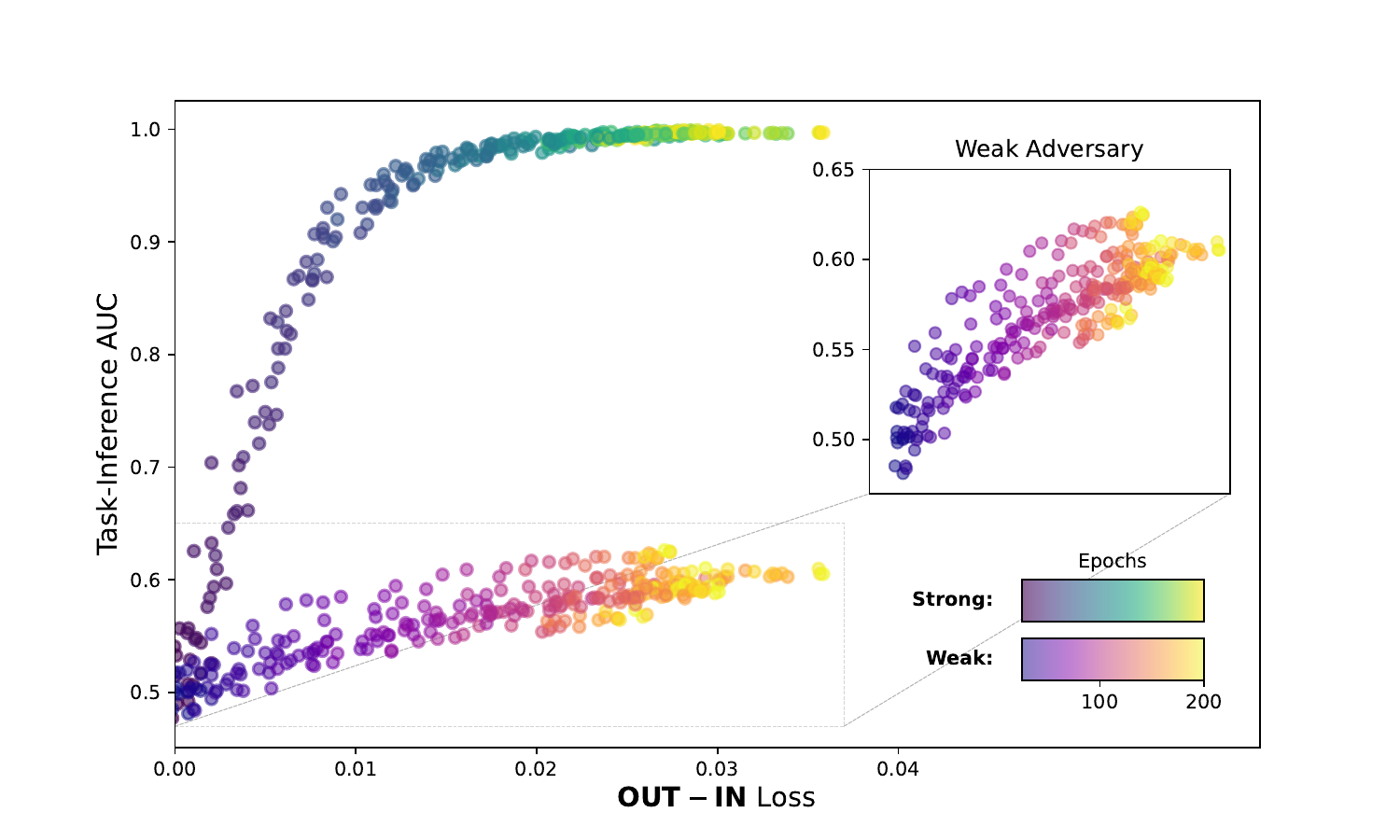}
        \caption{\textbf{Strong} and \textbf{Weak} (Personalization)}
        \label{fig:gengap_pers}
    \end{subfigure}
    \hfill
    \begin{subfigure}[b]{0.39\linewidth}
        \centering
        \includegraphics[width=\textwidth]{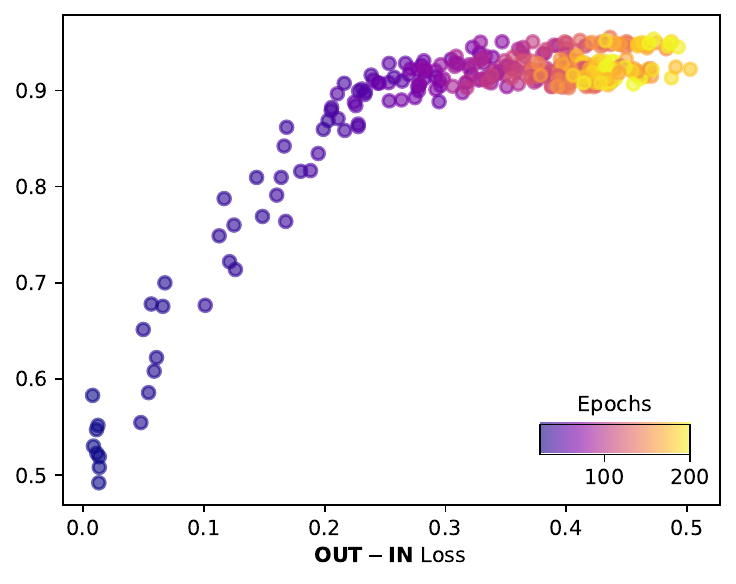}
        \caption{\textbf{Weak} (Multiple Learning Problems)}
        \label{fig:weak_gengap_task_level}
    \end{subfigure}
    \caption{Relationship Between Generalization Gaps and Inner Product Attack AUC on Stack Overflow}
    \label{fig:gengaps}
\end{figure*}

In contrast, when we run the exact same experiment with the \textbf{weak} adversary, we see that the attack AUC also grows as a function of the loss gap, but there is more variance. Here, the x-axis is the gap between the "zero-shot" loss and the loss on samples which did \emph{not} appear in training, but belong to training tasks. This observation implies that the model is both generalizing more effectively to tasks, or data distributions, that were seen in training and memorizing some amount of information about the task as a whole, rather than only particular training samples. A similar observation was made for user-inference attacks on generative LLMs in~\cite{kandpal2023userinference}, where the authors observe increasing attack AUC on user as a function of the model's validation loss on \textbf{IN} users. We supplement these findings by running another experiment, shown in Figure~\ref{fig:weak_gengap_task_level}, where the MTL model is trained to solve multiple learning problems at once. In this setting, the labeling of the training data for any task head directly corresponds to the task itself (e.g. the task-specific layer for "Python" only has positively labeled training samples that correspond to posts with the topic "Python"). Our observations highlight that the interdependence between the definition of tasks in an instance of MTL and the learning problems being solved by the task-specific layers has a measurable impact on task-inference AUC as the model becomes well-generalized on tasks included in training.

\section{Conclusion and Discussion} \label{sec:discussion}

In this work, we study privacy leakage when learning models over a mixture of tasks in multitask learning. In particular, we focus on the setting where MTL is used to jointly train a model for many tasks at once by learning a shared representation that captures common features between tasks. We propose a novel, purely black-box \emph{task-inference} threat model, where the adversary's goal is to infer the inclusion of a target task in training given only query access to the smallest shared component when learning jointly over many tasks, the shared representation. By analyzing our task-inference threat model in the context of tracing attacks on Gaussian mean estimation, we find a separation between \textbf{strong} adversaries with access to training samples and \textbf{weak} adversaries with access to fresh samples from the target task's distribution that did not appear in the training dataset. To verify our analysis, we propose purely black-box attacks on machine learning models trained with MTL and we conduct extensive experimentation in the vision and language domains for multiple use cases of MTL. We find that our attacks can consistently achieve non-trivial success rates in terms of AUC and TPR, even when the adversary has no reference knowledge and chooses thresholds based on percentiles. Additionally, we attempt to understand the factors that lead to task-inference leakage by measuring how attack AUC varies for both the \textbf{strong} and \textbf{weak} adversaries as a function of the model's generalization gap and running ablation studies over different MTL hyperparameters, such as the embedding dimension and number of tasks.

\subsection{Potential Defenses}
One potential defense against our attack could be user-level (and group-level) differential privacy~\cite{dwork-2006-dmns}, where neighboring datasets are defined by the inclusion or omission of a user's entire contribution to the dataset, rather than an individual sample. While there are several works study user-level differential privacy when training machine learning models~\cite{levy-2021-learninguserdp, chua-2024-privacyunituserlevel, charles-2024-finetuningllmuserlevel} and estimating high dimensional means~\cite{cummings-2022-userlevelmean, agarwal-2025-personlevel}, only one work develops algorithms with \emph{client-level} privacy guarantees~\cite{hu2023privatemtl} when MTL is applied to collaborative learning. In contrast to observations for empirical defenses made in the membership-inference literature~\cite{carlini-2022-lira}, our attack succeeds even when the target MTL model is trained using gradient clipping and normalization. In fact, we use gradient clipping in our evaluation to ensure that the MTL models converge smoothly and do not overfit when learning over very few samples from each task.


\section*{Acknowledgments}

This research was supported by NSF awards CNS-2247484 and CNS-2232692. A.S.'s work at BU was funded in part by NSF award CNS-2232694. In their work at Google, both M.J. and A.S. contributed only in an advisory capacity and did not train models or access datasets; all data was processed at Northeastern University.

\printbibliography

\appendix

\section{Models and Datasets}

In this section, we provide additional details for the MTL model training loop described in Section~\ref{alg:mtl_training}. 

\subsection{MTL Training} \label{appdx:training}

We use the Adam~\cite{kingma-2015-adam} and AdamW~\cite{loshchilov-2017-adamw} optimizers with weight decay to perform MTL training. Because each task could have only a small number of samples, we apply aggressive regularization to the task-specific layers to ensure that the final layers do not overfit to low-quality embeddings during training. When training our vision models where the total number of samples is small, we induce more regularization on top of weight decay by applying standard augmentations to the data such as random crops, rotations, horizontal flips, translations, and scaling. 

In our vision experiments, the CelebA dataset has very few samples per task-specific layer; 22 per celebrity after holding out samples for the weak adversary. Thus, we do two iterations of "warm start" training~\cite{nguyen-2023-flwarmup}, where we freeze the shared representation and optimize the task-specific layers for 40 epochs with with a low learning rate (e.g. $\eta = 10^{-4}$).  In our language experiments where we use BERT models, we do 20 epochs of warm start training. We also apply a linear projection on the typical, high-dimensional penultimate layer of our BERT and ResNet models to 16 or 32 dimensions in order to reduce the chance of task-specific layers overfitting. Each training step in MTL is performed over a minibatch of tasks, rather than the entire accumulated gradient of the dataset. We additionally use gradient clipping ($C = 1$), gradient normalization across tasks~\cite{chen-2018-gradnormalization}, and a smaller weight decay parameter (e.g. $\lambda = 10^{-4}$) than the task-specific layers (e.g. $\lambda = 10^{-3}$) to regularize the shared layers. In total, we train for 200-300 epochs, or communication rounds, for all of our models, passing through all of the tasks in the dataset each time.

\section{Proofs for Section~\ref{sec:mean_est}} \label{sec:proofs}

In this section, we provide the proofs for our theorems in Section~\ref{sec:mean_est}. Because the random variables in our estimation and attack are nested, we use the subscript $X$ to denote taking probability over sampling the data and subscript $\mu$ to denote taking probability over sampling tasks.
\newcommand{\mubar}[0]{\bar{\mu}}
\newcommand{\muhat}[0]{\hat{\mu}}
\newcommand{\muB}[0]{\mu_B}

\newtheorem*{strongadv}{Theorem~\ref{thm:strong_adv}}
\begin{strongadv}[Strong Adversary]  

Let $\tau$ be the index of the target task and suppose that the challenger sends the adversary a challenge set of $k$ samples $B$ such that, when the task is \textbf{IN}, the $k$ samples are drawn uniformly at random from $X_\tau$. Then the expected value of the statistic, $z$, when $\mu_\tau$ is \textbf{OUT} is 
    \[
        \ex{\mu, \: X}{z_{\textit{OUT}}} = 0
    \]
    and when  $\mu_\tau$ is \textbf{IN}, the expected value of $z$ is
    \[
        \ex{\mu, X}{z_{\textit{IN}}} = \frac{d}{T} \Big( \bar{\sigma}^2 + \frac{\sigma^2}{N} \Big)
    \]
\end{strongadv} 

\begin{proof}
In the \textbf{strong} case, for tasks that were included, the adversary has access to samples which were used to compute the mean, $\muhat$. Suppose that the strong adversary computes $z$, then, when $\mu_\tau$ is \textbf{OUT},
\begin{align*}
    \ex{\mu, X}{z_{\textit{OUT}}} &= \ex{\mu, X}{\inn{\muhat - \mubar}{\muB - \mubar}} \\
    &= \inn{\ex{\mu, X}{\muhat - \mubar}}{\ex{\mu, X}{\muB - \mubar}} \\ 
    &= 0
\end{align*}

In contrast, when the batch of points is \textbf{IN}, 
\begin{align*}
    \ex{\mu, X}{z_{\textit{IN}}} &= \ex{\mu, X}{\inn{\muhat- \mubar}{\muB - \mubar}} \\ 
    &= \sum_{i=1}^{d} \ex{\mu, X}{(\muhat- \mubar)_i \cdot (\muB - \mubar)_i} \\ 
\end{align*}

\noindent For succinctness, we drop the summation over $d$ dimensions as they are i.i.d. 
\begin{align*}
    \ex{\mu, X}{(\muhat - \mubar) \cdot (\muB - \mubar)} &= \ex{\mu, X}{\muhat\muB - \mubar\muB - \mubar\muhat + \mubar^{2}} \\
    &= \ex{\mu, X}{\muhat \muB} - \mubar \ex{\mu, X}{\muB} - \mubar \ex{\mu, X}{\muhat} + \mubar^{2} \\ 
    &= \ex{\mu, X}{\muhat \muB} - \mubar^2 
\end{align*}

\noindent Now, expanding $\muhat$ and $\muB$
\begin{align*}
\ex{\mu, X}{\muhat \muB} - \mubar^2  &= \ex{\mu, X}{\left( \frac{1}{T} \sum_{i=1}^{T}  \frac{1}{N} \sum_{j=1}^{N} X_{i,j} \right) \left( \frac{1}{k} \sum_{j=1}^{k} X_{\tau , j} \right) } - \mubar^2 
\end{align*}

We assume that $k \leq N$. Separating the correlated and uncorrelated tasks, we have
\begin{align*}
    &= \ex{\mu, X}{ \left( \frac{1}{T} \sum_{i \neq \tau}  \frac{1}{N} \sum_{j=1}^{N} X_{i,j} \right) \left( \frac{1}{k} \sum_{j=1}^{k} X_{\tau , j} \right) }  \\ &+   
    {\ex{\mu, X}{\left( \frac{1}{TN} \sum_{j=1}^{N} X_{\tau,j} \right) \left( \frac{1}{k} \sum_{j=1}^{k} X_{\tau,j} \right)} } - \mubar^2  \\ 
    &= \ex{\mu, X}{ \left( \frac{1}{T} \sum_{i \neq \tau}  \frac{1}{N} \sum_{j=1}^{N} X_{i,j} \right)} \cdot \ex{\mu, X}{\left( \frac{1}{k} \sum_{j=1}^{k} X_{\tau , j} \right) } \\ &+  
    {\ex{\mu, X}{\left( \frac{1}{TN} \sum_{j=1}^{N} X_{\tau,j} \right) \left( \frac{1}{k} \sum_{j=1}^{k} X_{\tau,j} \right)} } - \mubar^2  \\ 
    &= \frac{T-1}{T} \mubar^2 + {\ex{\mu, X}{\left( \frac{1}{TN} \sum_{j=1}^{N} X_{\tau,j} \right) \left( \frac{1}{k} \sum_{j=1}^{k} X_{\tau,j} \right)} } - \mubar^2  \\ 
    &=  \frac{1}{TNk} \left( {\sum_{j=1}^{k}\ex{\mu, X}{ X_{\tau,j}^{2}}  + \sum_{j \neq \ell} \ex{\mu, X}{ X_{\tau,j} \cdot X_{\tau, \ell} } }  -  Nk \mubar^2  \right)
\end{align*}

Now, taking the expectation over sampling the data then taking the expectation over sampling tasks,
\begin{align*}
    &=  \frac{1}{TNk} \left( {\sum_{j=1}^{k}\ex{\mu}{ \mu_{\tau}^{2} + \sigma^{2} }  + \sum_{j \neq \ell} \ex{\mu}{ \mu_{\tau}^{2} } }  -  Nk \mubar^2  \right)  \\ 
    &=  \frac{1}{TNk} \left(  k (\mubar^{2} + \bar{\sigma}^{2} + \sigma^{2})  + (Nk - k)  (\mubar^{2} + \bar{\sigma}^{2})    -  Nk \mubar^2  \right) \\ 
    &=  \frac{1}{TNk} \left(    k \bar{\sigma}^{2} + k\sigma^{2}  + (Nk - k)  \bar{\sigma}^{2}    \right)  \\ 
    &=  \frac{1}{TNk} \left(      Nk \bar{\sigma}^{2} + k\sigma^{2}    \right)  \\ 
    &=  \frac{\bar{\sigma}^{2}}{T} + \frac{\sigma^{2}}{NT}
\end{align*}

Summing over the $d$ i.i.d. dimensions, we get 
\[
\ex{\mu, X}{z_{\textit{IN}}} = \frac{d}{T} (\bar{\sigma}^{2} + \frac{\sigma^{2}}{N})
\]

\end{proof}

\newtheorem*{weakadv}{Theorem~\ref{thm:weak_adv}}
\begin{weakadv}[Weak Adversary] 
    Let $\tau$ be the index of the target task and suppose the challenger sends the adversary a challenge set of $k$ samples $B$ such that, when the task is \textbf{IN}, the $k$ samples are drawn i.i.d. from the same distribution as $X_\tau$, $\mathcal{N}(\mu_\tau, \sigma^2 \mathbb{I}_d)$. Then the expected value of the statistic, $z$, when $\mu_\tau$ is \textbf{OUT} is 0. When $\mu_\tau$ is \textbf{IN}, the expected value of $z$ is 
    \[
        \ex{\mu, X}{z_{\textit{IN}}} = \frac{d}{T} \bar{\sigma}^2
    \]
\end{weakadv}

\begin{proof}

The proof for the \textbf{OUT} case is identical to the proof for Theorem~\ref{thm:strong_adv}. When the \textbf{weak} adversary's challenge batch is \textbf{IN} 

\begin{align*}
    \ex{\mu, X}{z_{\textit{IN}}} &= \ex{\mu, X}{\inn{\muhat- \mubar}{\muB - \mubar}} \\ 
    &= \sum_{i=1}^{d} \ex{\mu, X}{(\muhat- \mubar)_i \cdot (\muB - \mubar)_i}
\end{align*}

\noindent For succinctness, we drop the summation over $d$ dimensions as they are i.i.d. 
\begin{align*}
    \ex{\mu, X}{(\muhat - \mubar) \cdot (\muB - \mubar)} &= \ex{\mu, X}{\muhat\muB - \mubar\muB - \mubar\muhat + \mubar^{2}} \\
    &= \ex{\mu, X}{\muhat \muB} - \mubar \ex{\mu, X}{\muB} - \mubar \ex{\mu, X}{\muhat} + \mubar^{2} \\ 
    &= \ex{\mu, X}{\muhat \muB} - \mubar^2 
\end{align*}

\noindent Expanding $\muhat$ and $\muB$ yields
\begin{align*}
\ex{\mu, X}{\muhat \muB} - \mubar^2  &= \ex{\mu, X}{\left( \frac{1}{T} \sum_{i=1}^{T}  \frac{1}{N} \sum_{j=1}^{N} X_{i,j} \right) \left( \frac{1}{k} \sum_{j=1}^{k} X_{\tau , j} \right) } - \mubar^2 
\end{align*}

We assume that $k \leq N$. Once again separating the correlated and uncorrelated tasks, we have
\begin{align*}
    ={} &\ex{\mu, X}{ \left( \frac{1}{T} \sum_{i \neq \tau}  \frac{1}{N} \sum_{j=1}^{N} X_{i,j} \right)} \cdot \ex{\mu, X}{\left( \frac{1}{k} \sum_{j=1}^{k} X_{\tau , j} \right) } \\ 
    &+  {\ex{\mu, X}{\left( \frac{1}{TN} \sum_{j=1}^{N} X_{\tau,j} \right) \left( \frac{1}{k} \sum_{j=1}^{k} X_{\tau,j} \right)} } - \mubar^2  \\ 
    ={} &\frac{T-1}{T} \mubar^2 + {\ex{\mu, X}{\left( \frac{1}{TN} \sum_{j=1}^{N} X_{\tau,j} \right) \left( \frac{1}{k} \sum_{j=1}^{k} X_{\tau,j} \right)} } - \mubar^2 
\end{align*}

Now, we note that the challenge batch samples are fresh, i.i.d samples from the task distribution $\mathcal{N}(\mu_\tau , \sigma^{2} \mathbb{I}_d)$. Thus, the challenge batch and $\hat{\mu}$ are independent over sampling the data, but correlated over tasks. Taking the expectation over both data and tasks,
\begin{align*}
    &= \frac{T-1}{T} \mubar^2 + \underset{\mu}{\mathbb{E}} \left[ \frac{1}{T} \underset{X}{\mathbb{E}}  \left[ \frac{1}{N} \sum_{j=1}^{N} X_{\tau,j} \right]  \cdot \underset{X}{\mathbb{E}}  \left[ \frac{1}{k} \sum_{j=1}^{k} X_{\tau,j} \right]  \right]  - \mubar^2  \\
    &= \frac{T-1}{T} \mubar^2 + {\frac{1}{T} \ex{\mu}{\mu_{\tau}^{2}} } - \mubar^2  \\
    &= \frac{T-1}{T} \mubar^2 + {\frac{1}{T} (\mubar^{2} + \bar{\sigma}^2) } - \mubar^2  \\
    &=   \frac{1}{T} (\mubar^{2} + \bar{\sigma}^2) - \frac{1}{T}\mubar^2 \\ 
    &=   \frac{\bar{\sigma}^2}{T} 
\end{align*}

Summing over the $d$ i.i.d. dimensions, we get that under the \textbf{weak} adversary
\[
\ex{\mu, X}{z_{\textit{IN}}} = \frac{d}{T} \bar{\sigma}^{2} 
\]

\end{proof}

\newtheorem*{vars}{Theorem~\ref{thm:variance} (Variance)}
\begin{vars}
    The variance of $z$ when $\mu_\tau$ is \textbf{OUT} is
    \[
        \var{\mu, X}{z_{\textit{OUT}}} = \frac{d}{T} \Big[ \bar{\sigma}^4 + \frac{\sigma^4}{kN} + \Big( \frac{k+N}{kN} \Big) ( \bar{\sigma}^2\sigma^2 ) \Big]
    \]
    \\
    When $\mu_\tau$ is \textbf{IN}, 
    \[
        \var{\mu,X}{{z}_{\textit{IN}}} \leq {3} \var{\mu,X}{{z}_{\textit{OUT}}}
    \]
\end{vars}

\begin{proof}
First, we define the following random variables:
\[
\alpha = \muhat - \mubar; \:
\alpha \sim \mathcal{N} \left( \vec{0}, \: \left( \frac{\bar{\sigma}^{2}}{T} + \frac{\sigma^{2}}{NT} \right) \mathbb{I}_d \right)
\]

\[
\beta = \mu_B - \mubar; \:
\beta \sim \mathcal{N} \left( \vec{0}, \: \left( \bar{\sigma}^{2} + \frac{\sigma^{2}}{k} \right) \mathbb{I}_d \right)
\]

Then, the variance in the \textbf{OUT} case is
\begin{align*}
    \var{\mu, X}{z_{\textit{OUT}}} &= \var{\mu, X}{\inn{\alpha}{\beta}} \\ 
    &= \sum_{i=1}^{d} \var{\mu, X}{\alpha_i \cdot\beta_i}
\end{align*}

For succinctness, we drop the summation with index $i$ over the i.i.d dimensions of the random variables. Then, we have
\begin{align*}
    \var{\mu, X}{\alpha \cdot \beta} &= \ex{\mu, X}{\alpha^{2} \beta^{2}} - \ex{\mu, X}{\alpha \beta}^{2} \\
    &= \ex{\mu, X}{\alpha^{2} \beta^{2}} \\ 
    &= \ex{\mu, X}{\alpha^{2}} \ex{\mu,X}{\beta{2}} \\ 
    &= \var{\mu, X}{\alpha} \var{\mu, X}{\beta} \\
    &= \left( \frac{\bar{\sigma}^{2}}{T} + \frac{\sigma^{2}}{NT} \right) \cdot \left( \bar{\sigma}^{2} + \frac{\sigma^{2}}{k} \right)  \\ 
    &= \frac{1}{T} \left[ \bar{\sigma}^4 + \frac{\sigma^4}{kN} + \Big( \frac{k+N}{kN} \Big) ( \bar{\sigma}^2\sigma^2 ) \right]
\end{align*}

Summing over the $d$ dimensions, we have 

\[
\var{\mu, X}{z_{\textit{OUT}}} = \frac{d}{T} \left[ \bar{\sigma}^4 + \frac{\sigma^4}{kN} + \Big( \frac{k+N}{kN} \Big) ( \bar{\sigma}^2\sigma^2 ) \right]
\]

Now, we can bound the variance of $z_{\textit{IN}} = \inn{\alpha}{\beta}$
\begin{align*}
     \var{\mu, X}{z_{\textit{IN}}} &= \var{\mu, X}{\inn{\alpha}{\beta}} \\
     &= \sum_{i=1}^{d} \var{\mu, X}{\alpha_i \cdot\beta_i} \\ 
     &= \sum_{i=1}^{d} \ex{\mu, X}{(\alpha_{i} \cdot \beta_{i})^{2}} - \ex{\mu, X}{\alpha_{i} \cdot \beta_{i}}^{2} \\ 
     &= \sum_{i=1}^{d} \ex{\mu, X}{(\alpha_{i} \cdot \beta_{i})^{2}} - \cov{\mu, X}{\alpha_i , \beta_i}^{2} \\
     &= \sum_{i=1}^{d} \ex{\mu, X}{(\alpha_{i} \cdot \beta_{i})^{2}} - \rho^{2} \var{\mu, X}{\alpha_i} \var{\mu, X}{\beta_i}
\end{align*}

\noindent where $\rho$ is the correlation coefficient between $\alpha_i$ and $\beta_i$. Now, using the Cauchy-Schwarz inequality and the fact that $\rho^{2} \geq 0$
\begin{align*}
     &\leq \sum_{i=1}^{d} \left(\sqrt{ \ex{\mu, X}{\alpha_{i}^{4}} \ex{\mu, X}{\beta_{i}^{4}} } \right) - \rho^{2} \var{\mu, X}{\alpha_i} \var{\mu, X}{\beta_i} \\
     &= \sum_{i=1}^{d} \left(\sqrt{3\var{\mu, X}{\alpha_{i}}^{2} \cdot 3\var{\mu, X}{\beta_{i}}^{2}}\right) - \rho^{2} \var{\mu, X}{\alpha_i} \var{\mu, X}{\beta_i} \\
     &= \sum_{i=1}^{d} 3\var{\mu, X}{\alpha_{i}} \var{\mu, X}{\beta_{i}} - \rho^{2} \var{\mu, X}{\alpha_i} \var{\mu, X}{\beta_i} \\
     &\leq \sum_{i=1}^{d} 3\var{\mu, X}{\alpha_{i}} \var{\mu, X}{\beta_{i}} \\
     &= 3\var{\mu, X}{z_{\textit{OUT}}}
\end{align*}

\newcommand{\indep}[0]{\perp\!\!\!\perp}
\newcommand{\alphaind}[0]{\alpha_{i,\indep}}
\newcommand{\alphatau}[0]{\alpha_{i,\tau}}
\newcommand{\alphatausplit}[1]{\alpha_{i,\tau,\textbf{#1}}}

We note that in this proof, the variance of $z_{\textit{IN}}$ is bounded by setting the squared correlation equal to 0. In practice, $\rho^{2} > 0$ since we know from Theorems~\ref{thm:strong_adv}  and \ref{thm:weak_adv} that $\ex{\mu,X}{\alpha_i \cdot \beta_i} =  \frac{\bar{\sigma}^2}{T}$ in the \textbf{weak} case and $\frac{\bar{\sigma}^2}{T} + \frac{\sigma^2}{NT}$ in the \textbf{strong} case.

\end{proof}

\section{Evaluation} \label{appdx:eval}

Here, we provide additional figures and results for Section~\ref{sec:eval}

\subsection{Variance Attack Test Statistic} \label{appdx:variance_attack}

We investigate the discrepancy in the ordering of our coordinate-wise variance test statistic when attacking MTL models trained for personalization. After MTL training, we compute the pairwise inner product, or similarity, of the weights in the task-specific layers for our CelebA, FEMNIST, and Stack Overflow models. Figure~\ref{fig:hist_taskheads} shows the distributions of the task head inner products, where we see that the vision models produce embeddings that lead task-specific layers to be correlated on average (i.e. inner product not centered around 0).

\begin{figure}[]
    \centering
    \includegraphics[width=0.4\linewidth]{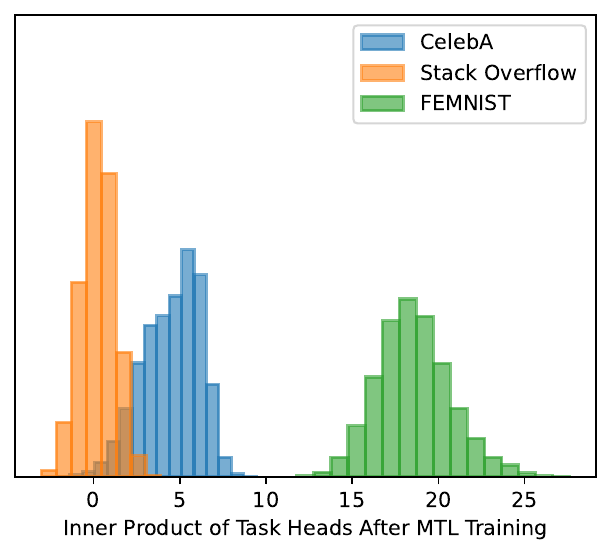}
    \caption{Task-Specific Layer Similarity (Personalization)}
    \label{fig:hist_taskheads}
\end{figure}



\subsection{Ablations on Synthetic Data} \label{appdx:synth}

We attempt to understand task-inference leakage by running our attacks on a synthetic dataset. The data generation process for this dataset is identical to the mean estimation example in Section~\ref{sec:mean_est}, but we adapt it for machine learning. To create the synthetic dataset, we start by sampling $2T$ i.i.d. tasks, $\{ \mu_1 \dots \mu_{2T} \}$ from a $d$-dimensional Gaussian distribution. For each of the tasks, $\mu_i$, we sample a dataset of $N$ $d$-dimensional vectors from the corresponding task distribution. To label the data for a learning problem, we first generate a random projection matrix $H \in \mathbb{R}^{k \times d}$, where $d$ is the data dimension and $k$ is the embedding dimension. Then, we randomly sample the "target" task-specific layers $g_1, \dots g_T \in \mathbb{R}^k$ and get the label for each sample, $\vec{x}$, coming from task $\tau \in [T]$ to be the sign of the inner product between $g_\tau$ and the embedding vector $H\vec{x}$ (that is, for any sample $X_{i,j}$, $y_{i,j} = \texttt{sign}(\inn{\: g_i}{ H X_{i,j} \:} )$.  This particular dataset is well suited for MTL as all tasks have unique labeling functions, but a common projection into the embedding space.

In our experiments on synthetic data, we use a simple neural network with 512 hidden units and a linear projection layer into the embedding space to approximate $H$. We vary the embedding dimension, number of samples per task, and number of tasks in the dataset, then measure our coordinate-wise variance attack's AUC over 4 training runs, with 64 trials of each attack on 8 random samples from each task. The results of this experiment are shown in Figure~\ref{fig:synthetic}. We observe that increasing the number of total samples in the dataset, whether by increasing the number of samples per task or the total number of tasks, has a sharp impact on the \textbf{strong} adversary's AUC. As the model has more samples to learn from, the gap between the \textbf{strong} and \textbf{weak} adversary's AUC scores shrinks. We also find that, in this suite of experiments, the embedding dimension has little impact on task-inference AUC.

\begin{figure*}[h]
    \centering
    \begin{subfigure}[b]{0.3\textwidth}
        \centering
        \includegraphics[width=\textwidth]{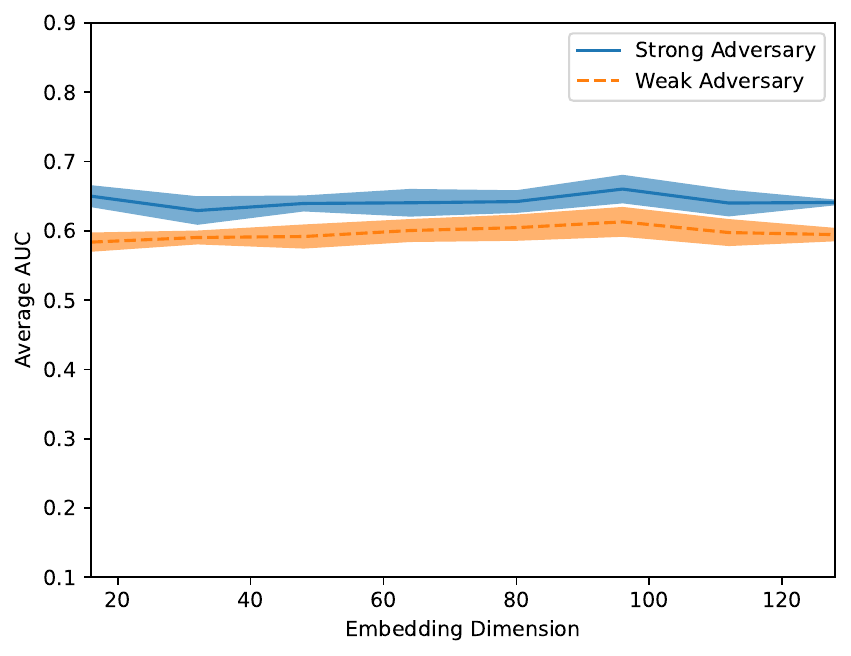}
        \caption{Embedding Dimension}
        \label{fig:embedding_dim}
    \end{subfigure}
    \hfill
    \begin{subfigure}[b]{0.3\textwidth}
        \centering
        \includegraphics[width=\textwidth]{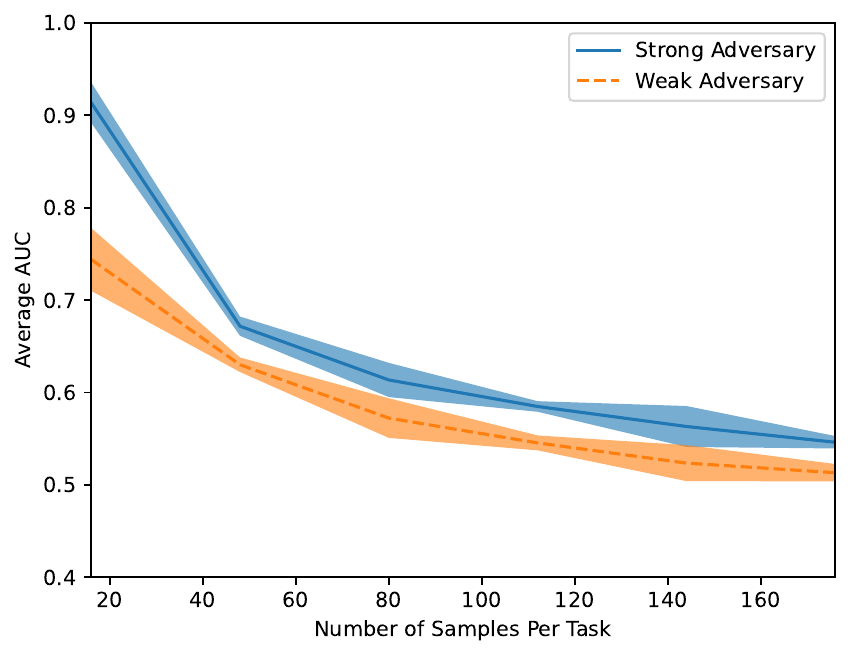}
        \caption{Number of Samples Per Task}
        \label{fig:num_samples}
    \end{subfigure}
    \hfill
    \begin{subfigure}[b]{0.3\textwidth}
        \centering
        \includegraphics[width=\textwidth]{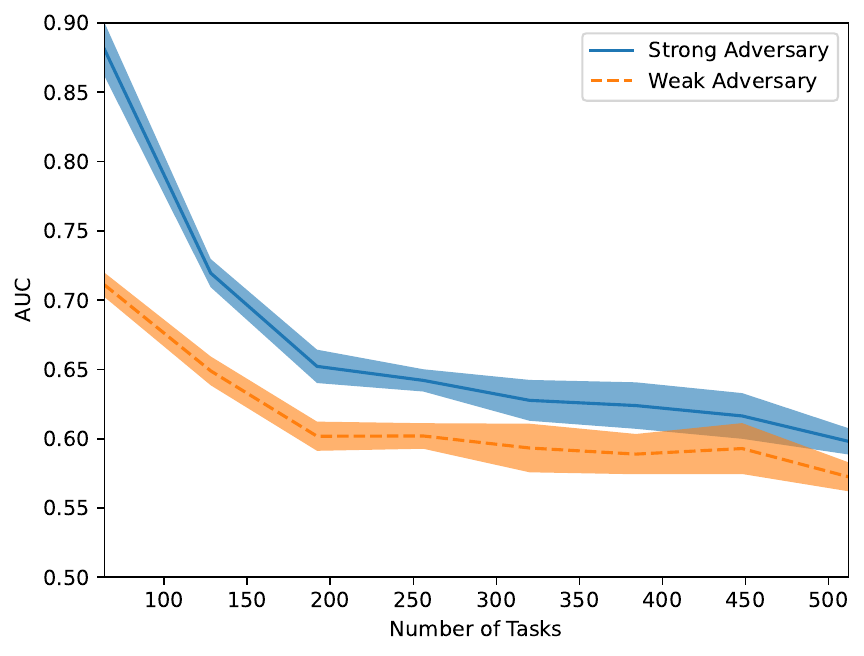}
        \caption{Number of Tasks}
        \label{fig:num_tasks}
    \end{subfigure}
    \caption{Ablations on Synthetic Data}
    \label{fig:synthetic}
\end{figure*}



\end{document}